\def\isarxivversion{1} 
\ifdefined\isarxivversion

\documentclass{article}
\usepackage{geometry}
\usepackage{amsmath, amsfonts, bm}
\usepackage{amsthm,graphicx}
\usepackage[dvipsnames]{xcolor}
\usepackage{mathtools}
\usepackage[normalem]{ulem}

\allowdisplaybreaks

\usepackage[unicode=true,
 bookmarks=false,
 breaklinks=false,pdfborder={0 0 1},colorlinks=false]
 {hyperref}
\hypersetup{
 colorlinks,citecolor=blue,filecolor=blue,linkcolor=blue,urlcolor=blue}

\else
\documentclass[anon,12pt]{colt2021}
\fi

\usepackage[linesnumbered,ruled,vlined]{algorithm2e}
\usepackage{array}
\usepackage{amssymb}
\usepackage{multirow}
\usepackage{color}
\usepackage[english]{babel}
\usepackage{graphicx}
\usepackage{natbib}
\usepackage{wrapfig}
\usepackage{epstopdf}
\usepackage{url}
\usepackage{graphicx}
\usepackage{color}
\usepackage{epstopdf}
\usepackage[T1]{fontenc}
\usepackage{bbm}
\usepackage{comment}
\usepackage{tikz}
\usepackage{booktabs,caption}
\usepackage[flushleft]{threeparttable}

\usepackage{tikz}
\usepackage{hyperref}
\hypersetup{colorlinks=true,citecolor=blue,linkcolor=red}

\usetikzlibrary{arrows}

\ifdefined\isarxivversion
\else 

\fi

\definecolor{mygreen}{RGB}{80,180,0}

\definecolor{dacong}{RGB}{10,103,68}

\ifdefined\isarxivversion
\usepackage{amsmath}
\usepackage{amsthm}
\usepackage{subfig}

\usepackage{enumitem}

\newtheorem{theorem}{Theorem}
\newtheorem{lemma}{Lemma}

\newtheorem{remark}{Remark}

\else
\fi

\newtheorem{assumption}{Assumption}

\newcommand{\cS}{\mathcal{S}}
\newcommand{\cA}{\mathcal{A}}

\newcommand{\wh}{\widehat}
\newcommand{\wt}{\widetilde}

\newcommand{\eps}{\varepsilon}

\newcommand{\A}{\mathcal{A}}

\newcommand{\E}{\mathbb{E}}

\renewcommand{\tilde}{\wt}
\renewcommand{\hat}{\wh}

\newcommand{\KL}{\mathsf{KL}}

\newcommand{\mymid}{\,|\,}

\newcommand{\BO}{\mathcal{O}}

\newcommand{\epsopt}{\varepsilon_{\mathsf{opt}}}
\newcommand{\epseval}{\varepsilon_{\mathsf{eval}}}

\definecolor{yxc}{RGB}{255,0,0}
\definecolor{yjc}{RGB}{190,0,255}
\definecolor{whz}{RGB}{200,100,200}

\makeatletter
\newcommand*{\RN}[1]{\expandafter\@slowromancap\romannumeral #1@}
\makeatother

\title{}

\ifdefined\usebigfont

\usepackage{times}
\usepackage[fontsize=13pt]{scrextend}
\AtBeginDocument{\newgeometry{letterpaper,left=1.56in,right=1.56in,top=1.71in,bottom=1.77in}}
\else
\usepackage{scrextend}
\fi

\begin{document}

\ifdefined\isarxivversion
\title{Policy Mirror Descent for Regularized Reinforcement Learning: \\ A Generalized Framework with Linear Convergence}

\author{Wenhao Zhan\footnote{The first two authors contributed equally.} \thanks{Department of Electrical and Computer Engineering, Princeton University.}\\
Princeton University    \\
	\and
	Shicong Cen\footnotemark[1] \thanks{Department of Electrical and Computer Engineering, Carnegie Mellon University.}\\
	Carnegie Mellon University \\
	\and
	Baihe Huang\thanks{Department of Electrical Engineering and Computer Sciences, University of California, Berkeley.}\\
	 University of Berkeley \\
	\and 
	 Yuxin Chen\thanks{Department of Statistics and Data Science, Wharton School, University of Pennsylvania.}  \\
 University of Pennsylvania  \\
	\and
	  Jason D. Lee\footnotemark[2] \\
   Princeton University \\
	\and
	Yuejie Chi\footnotemark[3]\\
	Carnegie Mellon University \\
	}
\date{May 2021;~~ Final: January 2023}


\else
\fi

%

\maketitle

\begin{abstract}

Policy optimization, which finds the desired policy  by maximizing value functions via optimization techniques, 
lies at the heart of reinforcement learning (RL).
In addition to value maximization, other practical considerations arise as well, including the need of encouraging exploration, and that of ensuring certain structural properties of the learned policy due to safety, resource and operational constraints. 
These can often be accounted for via regularized RL, which augments the target value function with a structure-promoting regularizer.

Focusing on discounted infinite-horizon Markov decision processes, we propose a generalized policy mirror descent (GPMD) algorithm for solving regularized RL. 
As a generalization of policy mirror descent \citep{lan21}, our algorithm
accommodates a general class of convex regularizers and promotes the use of Bregman divergence in cognizant of the regularizer in use. 
We demonstrate that our algorithm converges linearly to the global solution over an entire range of learning rates, in a dimension-free fashion, even when the regularizer lacks strong convexity and smoothness.  
In addition, this linear convergence feature is provably stable in the face of inexact policy evaluation and imperfect policy updates. 
Numerical experiments are provided to corroborate the appealing performance of GPMD. 
\end{abstract}


\noindent \textbf{Keywords:} policy mirror descent, Bregman divergence, regularization, nonsmooth, policy optimization

\setcounter{tocdepth}{2}
\tableofcontents

\section{Introduction}

Policy optimization lies at the heart of recent successes of reinforcement learning (RL) \citep{mnih2015human}. 
In its basic form, the optimal policy of interest, or a suitably parameterized version,
is learned by attempting to maximize the value function in a Markov decision processes (MDP). 
For the most part, the maximization step is carried out by means of first-order optimization algorithms amenable to large-scale applications,   
 whose foundations were set forth in the early works of \citet{williams1992simple,sutton2000policy}. 
A partial list of widely adopted variants in modern practice includes
policy gradient (PG) methods \citep{sutton2000policy}, natural policy gradient (NPG) methods \citep{kakade2001natural},  TRPO \citep{schulman2015trust}, 
PPO \citep{schulman2017proximal}, soft actor-critic methods \citep{haarnoja2018soft}, to name just a few. 
In comparison with model-based and value-based approaches, 
this family of policy-based algorithms offers a remarkably flexible framework that accommodates both continuous and discrete action spaces, 
and lends itself well to the incorporation of powerful function approximation schemes like neural networks. 
In stark contrast to its practical success, however, 
theoretical understanding of policy optimization remains severely limited even for the tabular case,  
largely owing to the ubiquitous nonconvexity issue underlying the objective function.





\subsection{The role of regularization}


In practice, there are often competing objectives and additional constraints that the agent has to deal with in conjunction with maximizing values, 
which motivate the studies of regularization techniques in RL. In what follows, we isolate a few representative examples.  
\begin{itemize}
	\item {\em Promoting exploration.} In the face of large problem dimensions and complex dynamics, 
		it is often desirable to maintain a suitable degree of randomness in the policy iterates,
		in order to encourage exploration and discourage premature convergence to sub-optimal policies. 
		A popular strategy of  this kind is to enforce entropy regularization \citep{williams1991function}, which penalizes policies that are not sufficiently stochastic. Along similar lines, the Tsallis entropy regularization \citep{chow2018path,lee2018sparse} further promotes sparsity of the learned policy while encouraging exploration, ensuring that the resulting policy does not assign non-negligible probabilities to too many sub-optimal actions.


	\item {\em Safe RL.} In a variety of application scenarios such as industrial robot arms and self-driving vehicles, 
		the agents are required to operate safely both to themselves and the surroundings \citep{amodei2016concrete,moldovan2012safe}; 
		for example, certain actions might be strictly forbidden in some states. 
		One way to incorporate such prescribed operational constraints is through adding a regularizer (e.g., a properly chosen log barrier or indicator function tailored to the constraints) to explicitly account for the constraints.

	\item {\em Cost-sensitive RL.} In reality, different actions of an agent might incur drastically different costs even for the same state. 
		This motivates the design of new objective functions that properly trade off the cumulative rewards against the accumulated cost, 
		which often take the form of certain regularized value functions. 
\end{itemize}
Viewed in this light, 
it is of imminent value to develop a unified framework towards understanding the capability and limitations of regularized policy optimization. 
While a recent line of works \citep{agarwal2019optimality,mei2020global,cccwc20} have looked into specific types of regularization techniques such as entropy regularization, 
existing convergence theory remains highly inadequate when it comes to a more general family of regularizers.


\subsection{Main contributions}


The current paper focuses on policy optimization for regularized RL in a $\gamma$-discounted infinite horizon Markov decision process (MDP) with state space $\cS$, action space $\cA$, and reward function $r(\cdot,\cdot)$.
The goal is to find an optimal policy that maximizes a regularized value function. Informally speaking, the regularized value function associated with a given policy $\pi$ takes the following form:  
\[
	V_{\tau}^{\pi} = V^{\pi} - \tau \mathbb{E}\big[ h_s\big( \pi(\cdot \mymid s) \big) \big],
\]
where $V^{\pi}$ denotes the original (unregularized) value function, $\tau>0$ is the regularization parameter, 
$h_s(\cdot)$ denotes a convex regularizer employed to regularize the policy in state $s$, 
and the expectation is taken over certain marginal state distribution w.r.t.~the MDP (to be made precise in Section~\ref{sec:setting}). 
It is noteworthy that this paper does not require the regularizer $h_s$ to be either strongly convex or smooth. 

In order to maximize the regularized value function \eqref{eq:reg-Q-V-relation}, 
\citet{lan21} exhibited a seminal algorithm called {\em Policy Mirror Descent (PMD)}, 
which can be viewed as an adaptation of the mirror descent algorithm \citep{nemirovsky1983problem,beck2003mirror} to the realm of policy optimization. 
In particular, PMD subsumes the natural policy gradient (NPG) method \citep{kakade2001natural} as a special case. To further generalize PMD \citep{lan21},
we propose an algorithm called {\em Generalized Policy Mirror Descent (GPMD)}. In each iteration, the policy is updated for each state in parallel via a mirror-descent style update rule. 
In sharp contrast to \citet{lan21} that considered a generic Bregman divergence, our algorithm selects the Bregman divergence {\em adaptively} in cognizant of the regularizer, which leads to complementary perspectives and insights.
Several important features and theoretical appeal of GPMD are summarized as follows.

\begin{itemize}

	\item GPMD substantially broadens the range of (provably effective) algorithmic choices for regularized RL, and subsumes several well-known algorithms as special cases. For example, it reduces to regularized policy iteration \citep{geist2019theory} when the learning rate tends to infinity, and subsumes entropy-regularized NPG methods as special cases if we take the Bregman divergence to be the Kullback-Leibler (KL) divergence \citep{cccwc20}. 


\item Assuming exact policy evaluation and perfect policy update in each iteration, GPMD converges linearly---in a dimension-free fashion---
over the {\em entire} range of the learning rate $\eta >0$. 
More precisely, it converges to an $\varepsilon$-optimal regularized Q-function in no more than an order of
$$ \frac{1+\eta\tau}{\eta\tau(1-\gamma)}\log\frac{1}{\varepsilon}$$
iterations (up to some logarithmic factor). 
Encouragingly, this appealing feature is valid for a broad family of convex and possibly nonsmooth regularizers. 

\item The intriguing convergence guarantees are robust in the face of inexact policy evaluation and imperfect policy updates, namely,  
	the algorithm is guaranteed to converge linearly at the same rate until an error floor is hit.  See Section~\ref{sec:main-results-approx} for details. 

\item  Numerical experiments are provided in Section~\ref{sec:experiments} to demonstrate the practical applicability and appealing performance of the proposed GPMD algorithm. 

\end{itemize}

Finally, we find it helpful to briefly compare the above findings with prior works. 
As soon as the learning rate exceeds $\eta \geq 1/\tau$, the iteration complexity of our algorithm is at most on the order of $\tfrac{1}{1-\gamma}\log\frac{1}{\varepsilon}$, thus matching that of regularized policy iteration \citep{geist2019theory}. 
In comparison to \citet{lan21}, our work sets forth a different framework to analyze mirror-descent type algorithms for regularized policy optimization, generalizing and refining the approach in \citet{cccwc20} far beyond entropy regularization. When constant learning rates are employed, the linear convergence of PMD \citep{lan21} critically requires the regularizer to be strongly convex, with only sublinear convergence theory established for convex regularizers. 
In contrast, we establish the linear convergence of GPMD under constant learning rates even in the absence of strong convexity. Furthermore, for the special case of entropy regularization, the stability analysis of GPMD also significantly improves over the prior art in \citet{cccwc20}, preventing the error floor from blowing up when the learning rate approaches zero, as well as incorporating the impact of optimization error that was previously uncaptured. More detailed comparisons with \citet{lan21} and \citet{cccwc20} can be found in Section~\ref{sec:main-results}.

%

\subsection{Related works}

Before embarking on our algorithmic and theoretic developments, we briefly review a small sample of other related works.

\paragraph{Global convergence of policy gradient methods.} Recent years have witnessed a surge of activities towards understanding the global convergence properties of  policy gradient methods and their variants for both continuous and discrete RL problems, examples including \citet{fazel2018global,bhandari2019global,agarwal2019optimality,zhang2019policy,wang2019neural,mei2020escaping,br20,khodadadian2021linear,liu2020improved,mei2020escaping,agazzi2020global,xu2019sample,wang2019neural,cccwc20,mei2021leveraging,liu2019neural,wang2020global,zhang2020sample,zhang2021convergence,zhang2020variational,shani2019adaptive}, among other things. \citet{neu2017unified} provided the first interpretation of NPG methods as mirror descent \citep{nemirovsky1983problem}, thereby enabling the adaptation of techniques for analyzing mirror descent to the studies of NPG-type algorithms such as TRPO \citep{shani2019adaptive,tomar2020mirror}.  It has been shown that the NPG method converges sub-linearly for unregularized MDPs with a fixed learning rate \citep{agarwal2019optimality}, and converges linearly if the learning rate is set adaptively \citep{khodadadian2021linear}, via exact line search \citep{br20}, or following a geometrically increasing schedule \citep{xiao2022convergence}. The global linear convergence of NPG holds more generally for an arbitrary fixed learning rate when entropy regularization is enforced \citep{cccwc20}. 
Noteworthily, \cite{li2021softmax} established a lower bound indicating that softmax PG methods can take an exponential time---in the size of the state space---to converge, 
while the convergence rates of NPG-type methods are almost independent of the problem dimension. 
In addition, another line of recent works \citep{AYBBLSzW19, Hao2021-ig, LazicYAS21} established regret bounds for approximate NPG methods---termed as KL-regularized approximate policy iteration therein---for infinite-horizen undiscounted MDPs, 
which are beyond the scope of the current paper.

\paragraph{Regularization in RL.} Regularization has been suggested to the RL literature either through the lens of optimization \citep{dai2018sbeed,agarwal2019optimality}, or through the lens of dynamic programming \citep{geist2019theory,vieillard2020leverage}. Our work is clearly an instance of the former type. 
Several recent results in the literature merit particular attention: \citet{agarwal2019optimality} demonstrated sublinear convergence guarantees for PG methods in the presence of relative entropy regularization, \citet{mei2020global} established linear convergence of entropy-regularized PG methods, 
whereas \citet{cccwc20} derived an almost dimension-free linear convergence theory for NPG methods with entropy regularization. 
Most of the existing literature focused on the entropy regularization or KL-type regularization, and the studies of general regularizers had been quite limited until the recent work \citet{lan21}. 
The regularized MDP problems are also closely related to the studies of constrained MDPs,
as both types of problems can be employed to model/promote constraint satisfaction in RL, as recently investigated in, e.g., \citet{chow2018lyapunov,efroni2020exploration,ding2021provably,yu2019convergent,xu2020primal}. 
Note, however, that it is difficult to directly compare our algorithm with these methods, due to drastically different formulations and settings.

\subsection{Notation}

Let us introduce several notation that will be adopted throughout. 
For any set $\cA$, we denote by $\vert\mathcal{A}\vert$ the cardinality of a set $\mathcal{A}$, and  
let $\Delta(\mathcal{A})$ indicate the probability simplex over the set $\mathcal{A}$.
For any convex and differentiable function $h(\cdot)$,  the Bregman divergence generated by $h(\cdot)$ is defined as
\begin{align}
	D_h(z,x) \coloneqq h(z)-h(x)- \big\langle\nabla h(x),z-x \big\rangle.
	\label{eq:defn-Bregman-div-standard}
\end{align}
For any convex (but not necessarily differentiable) function $h(\cdot)$, we denote by $\partial h$ the subdifferential of $h$. 
Given two probability distributions $\pi_1$ and $\pi_2$ over $\cA$,  the KL divergence from $\pi_2$ to $\pi_1$ is defined as $\KL(\pi_1 \,\|\, \pi_2)  \coloneqq \sum_{a\in\cA}\pi_1(a) \log\frac{\pi_1(a)}{\pi_2(a)}$. 
For any vectors $a=[a_i]_{1\leq i\leq n}$ and $b=[b_i]_{1\leq i\leq n}$, the notation $a\leq b$ (resp.~$a\geq b$) means that $a_i\leq b_i$ ($a_i\geq b_i$) for every $1\leq i\leq n$. 
We shall also use $1$ (resp.~$0$) to denote the all-one (resp.~all-zero) vector whenever it is clear from the context. 

\section{Model and algorithms}
\label{sec:model}

\subsection{Problem settings}
\label{sec:setting}

\paragraph{Markov decision process (MDP).} 

The focus of this paper is a discounted infinite-horizon Markov decision process, as represented by $\mathcal{M}=(\mathcal{S},\mathcal{A},P,r,\gamma)$ \citep{bertsekas2017dynamic}. 
Here, $\mathcal{S}\coloneqq \{1,\cdots, |\cS|\}$ is the state space, $\mathcal{A}\coloneqq \{1,\cdots, |\cA|\}$ is the action space, $\gamma \in [0,1)$ is the discount factor, 
$P: \mathcal{S}\times\mathcal{A}\to\Delta(\mathcal{S})$ is the probability transition matrix (so that $P(\cdot \mymid s,a)$ is the transition probability from state $s$ upon execution of action $a$), 
whereas $r:\mathcal{S}\times\mathcal{A}\to[0,1]$ is the reward function (so that $r(s,a)$ indicates the immediate reward received in state $s$ after action $a$ is executed). 
Here, we focus on finite-state and finite-action scenarios, meaning that both $\vert\mathcal{S}\vert$ and $\vert\mathcal{A}\vert$ are assumed to be finite. 
A policy $\pi:\mathcal{S}\to\Delta(\mathcal{A})$ specifies a possibly randomized action selection rule, namely, $\pi(\cdot\mymid s)$ represents the action selection probability in state $s$.

For any policy $\pi$, we define the associated {\em value function} $V^{\pi}:\mathcal{S}\to\mathbb{R}$ as follows
\begin{equation}
	\forall s\in\mathcal{S}: 
	\qquad
	V^{\pi}(s):= \mathop{\mathbb{E}}\limits_{\substack{a_t\sim \pi(\cdot|s_t),\\ s_{t+1}\sim P(\cdot | s_t, a_t), ~\forall t\geq 0}}\left[\sum_{t=0}^{\infty}\gamma^tr(s_t,a_t) ~\Big\vert~ s_0=s\right], 
	\label{eq:defn-Vpi}
\end{equation}
which can be viewed as the utility function we wish to maximize. 
Here, the expectation is taken over the randomness of the MDP trajectory $\{(s_t,a_t)\}_{t\geq 0}$ induced by policy $\pi$. 
Similarly, when the initial action $a$ is fixed, we can define the {\em action-value function (or Q-function)} as follows
\begin{equation}
	\label{eq:defn-Qpi}
	\forall (s,a)\in\mathcal{S} \times \cA: 
	\qquad
	Q^{\pi}(s,a):= \mathop{\mathbb{E}}\limits_{\substack{s_{t+1}\sim P(\cdot | s_t, a_t), \\  a_{t+1}\sim \pi(\cdot|s_{t+1}),~\forall t\geq 0}}
	\left[\sum_{t=0}^{\infty}\gamma^tr(s_t,a_t) ~\Big\vert~ s_0=s, a_0 = a\right]. 
\end{equation}
As a well-known fact, the policy gradient of $V^{\pi}$ (w.r.t.~the policy $\pi$) admits the following closed-form expression (\citet{sutton2000policy}) 
\begin{align}
	\label{eq:policy-gradient}
	\forall(s,a)\in\cS\times \cA: \qquad 
	\frac{\partial V^{\pi}(s_0)}{ \partial \pi (a\mymid s)} = \frac{1}{1-\gamma} d_{s_0}^{\pi}(s)  Q^{\pi}(s,a).
\end{align}
Here, $d_{s_0}^{\pi}\in \Delta(\cS)$ is the so-called {\em discounted state visitation distribution} defined as follows
\begin{align}
	d_{s_0}^{\pi}(s) \coloneqq (1-\gamma) \sum_{t=0}^{\infty} \gamma^t \mathbb{P}^{\pi}( s_t = s \mymid s_0 ) ,
	\label{eq:defn-state-visitation}
\end{align}
where $ \mathbb{P}^{\pi}( s_t = s \mymid s_0 )$ denotes the probability of $s_t =s$ when the MDP trajectory $\{s_t\}_{t\geq 0}$ is generated under policy $\pi$ given the initial state $s_0$.

Furthermore, the optimal value function and the optimal Q-function are defined and denoted by
\begin{align}
	\forall (s,a)\in \cS\times \cA: \qquad
	V^{\star}(s) \coloneqq \max_{\pi} V^{\pi}(s), \qquad  Q^{\star}(s,a) \coloneqq \max_{\pi} Q^{\pi}(s,a). 
\end{align}
It is well known that there exists at least one optimal policy, denoted by $\pi^{\star}$, that simultaneously maximizes the value function and the Q-function for all state-action pairs \citep{agarwal2019reinforcement}.

\paragraph{Regularized MDP.} 
In practice, the agent is often asked to design policies that possess certain structural properties in order to be cognizant of system constraints such as safety and operational constraints, as well as encourage exploration during the optimization/learning stage. 
A natural strategy to achieve these is to resort to the following {\em regularized value function} w.r.t.~a given policy $\pi$ \citep{neu2017unified,mei2020global,cccwc20,lan21}:
\begin{align}
	\forall s\in\mathcal{S}: 
	\qquad
	V^{\pi}_{\tau}(s) &\coloneqq \mathop{\mathbb{E}}\limits_{\substack{a_t\sim \pi(\cdot|s_t),\\ s_{t+1}\sim P(\cdot | s_t, a_t), ~\forall t\geq 0}} 
	\left[\sum_{t=0}^{\infty}\gamma^t\Big\{ r(s_t,a_t)-\tau h_{s_t}\big(\pi(\cdot \mymid s_t)\big)\Big\} ~\Big\vert~ s_0=s\right] \notag\\
	&= V^{\pi}(s) - \frac{\tau}{1-\gamma} \sum_{s'\in \cS} d_s^{\pi}(s')   h_{s'}\big(\pi(\cdot \mymid s')\big), 
	\label{eq:defn-reg-value-function}
\end{align}
where $h_s: \Delta_{\zeta}(\mathcal{A})\to \mathbb{R}$ stands for a convex and possibly nonsmooth regularizer for state $s$,  $\tau >0$ denotes the regularization parameter,
and $d_s^{\pi}(\cdot)$ is defined in \eqref{eq:defn-state-visitation}. Here, for technical convenience, we assume throughout that  $h_s(\cdot)$ ($s\in \cS$) is well-defined over an ``$\zeta$-neighborhood'' of the probability simplex $\Delta(\cA)$ defined as follows  
\[
	\Delta_{\zeta} (\mathcal{A}) \coloneqq \left\{x=[x_a]_{a\in \cA} ~\Big\vert~ x_a \ge 0 \text{ for all } a \in \mathcal{A};~ 1- \zeta \le \sum_{a\in \mathcal{A}} x_a \le 1 + \zeta \right\} ,
\]
where $\zeta>0$ can be an arbitrary constant. 
For instance, entropy regularization adopts the choice $h_s(p)= \sum_{i\in \cA}p_i \log p_i$ for all $s\in \cS$ and $p\in\Delta(\cA)$, which coincides with the negative Shannon entropy of a probability distribution. Similar, a KL regularization adopts the choice $h_s(p) = \KL( p \,\|\, p_{\mathsf{ref}})$, which penalizes the distribution $p$ that deviates from the reference  $p_{\mathsf{ref}}$. As another example, a weighted $\ell_1$ regularization adopts the choice $h_s(p) = \sum_{i\in\cA} w_{s,i} p_i$ for all $s\in \cS$ and $p\in\Delta(\cA)$, where $w_{s,i}\geq 0$ is the cost of taking action $i$ at state $s$, and the regularizer $h_s(\pi(\cdot|s))$ captures the expected cost of the policy $\pi$ in state $s$.
Throughout this paper, we impose the following assumption. 
\begin{assumption}
\label{assumption:h-inf}
	Consider an arbitrarily small constant $\zeta>0$. For for any  $s\in \cS$, suppose that $h_s(\cdot)$ is convex and
\begin{align}
	h_s(p) = \infty \qquad \text{for any } p \notin \Delta_\zeta(\mathcal{A}) .
	\label{eq:hs-boundary-condition}
\end{align}
\end{assumption}
%


Following the convention in prior literature (e.g., \citet{mei2020global}), we also define the corresponding {\em regularized Q-function} as follows:
\begin{subequations}
\label{eq:relation-reg-Q-reg-V}
\begin{equation}
	\forall (s,a)\in\mathcal{S} \times \cA: 
	\qquad
	{Q}^{\pi}_{\tau}(s,a):=r(s,a)+\gamma\mathop{\mathbb{E}}\limits_{s'\sim P( \cdot | s,a)}\big[V^{\pi}_{\tau}(s')\big].
	\label{eq:defn-reg-Q-function}
\end{equation}
As can be straightforwardly verified, one can also express $V^{\pi}_{\tau}$ in terms of $Q^{\pi}_{\tau}$ as 
\begin{equation}
	\forall s\in\mathcal{S} : 
	\qquad
	V^{\pi}_{\tau}(s):=\mathop{\mathbb{E}}\limits_{a\sim\pi(\cdot |s)} \Big[ {Q}^{\pi}_{\tau}(s,a) - \tau h_s\big(\pi(\cdot \mymid s)\big) \Big].
	\label{eq:reg-Q-V-relation}
\end{equation}
\end{subequations}
The optimal regularized value function $V^{\star}_{\tau}$ and the corresponding optimal policy $\pi_{\tau}^{\star}$ are defined respectively as follows: 
\begin{equation}
\label{optimal policy}
	\forall s\in\mathcal{S}: \qquad V^{\star}_{\tau}(s) \coloneqq V^{\pi_{\tau}^\star}_{\tau}(s) = \max_{\pi}V^{\pi}_{\tau}(s), \qquad  
	\pi_{\tau}^{\star} \coloneqq \arg\max_{\pi}V^{\pi}_{\tau}.
\end{equation}
It is worth noting that \citet{puterman2014markov} asserts the {\em existence} of an optimal policy $\pi_{\tau}^{\star}$ that achieves (\ref{optimal policy})  simultaneously for all $s\in\mathcal{S}$. 
Correspondingly, we shall also define the resulting optimal regularized Q-function as
\begin{equation}
	\forall (s,a)\in \cS\times \cA: \qquad
	Q^{\star}_{\tau} (s,a) = Q^{\pi_{\tau}^\star}_{\tau} (s,a) .
\end{equation}
%



\subsection{Algorithm: generalized policy mirror descent}
\label{sec:algorithm-PMD}

Motivated by PMD \citep{lan21}, 
we put forward a generalization of PMD that selects the Bregman divergence in cognizant of the regularizer in use. A thorough comparison with \citet{lan21} will be provided after introducing our generalized PMD algorithm.

\paragraph{Review: mirror descent (MD) for the composite model.} 
To better elucidate our algorithmic idea, let us first briefly review the design of classical mirror descent---originally proposed by \citet{nemirovsky1983problem}---in the optimization literature.
Consider the following composite model: 
\[
	\text{minimize}_x \quad F(x) \coloneqq f(x) + h(x), 
\]
where the objective function consists of two components. The first component is assumed to be differentiable, while the second component $h(\cdot)$ can be more general and is commonly employed to model some sort of regularizers.  
To solve this composite problem, one variant of mirror descent adopts the following update rule (see also \citet{beck2017first,duchi2010composite}):
\begin{align}
	x^{(k+1)} = \arg\min_{x} \left\{ f\big(x^{(k)}\big) + \big\langle \nabla f(x^{(k)}), x \big\rangle + h(x) + \frac{1}{\eta} D_h\big( x, x^{(k)} \big) \right\} ,
	\label{eq:general-MD}
\end{align}
where $\eta>0$ is the learning rate or step size, and $D_h(\cdot , \cdot)$ is the Bregman divergence defined in \eqref{eq:defn-Bregman-div-standard}. 
Note that the first term  within the  curly brackets of \eqref{eq:general-MD} can be safely discarded as it is a constant given $x^{(k)}$. 
In words, the above update rule approximates $f(x)$ via its first-order Taylor expansion $f\big(x^{(k)}\big) + \big\langle \nabla f(x^{(k)}), x\big\rangle$ at the point  $x^{(k)}$, employs the Bregman divergence $D_h$ to monitor the difference between the new iterate and the current iterate $x^{(k)}$, 
and attempts to optimize such (properly monitored) approximation instead. While one can further generalize the Bregman divergence to $D_{\omega}$ for a different generator $\omega$, we shall restrict attention to the case with $h=\omega$ in the current paper.


\paragraph{The proposed algorithm.} 

We are now ready to present the algorithm we come up with, which is an extension of the PMD algorithm \citep{lan21}. 
For notational simplicity, we shall write
\begin{align}\label{eq:short_hand_v}
	V^{(k)}_{\tau} \coloneqq V^{\pi^{(k)}}_{\tau},
	\qquad
	Q^{(k)}_{\tau}(s,a) \coloneqq Q^{\pi^{(k)}}_{\tau}(s,a)
	\qquad \text{and}\qquad
	d^{(k)}_{s_0}(s) \coloneqq d^{\pi^{(k)}}_{s_0}(s)
\end{align}
throughout the paper, where $\pi^{(k)}$ denotes our policy estimate in the $k$-th iteration.

To begin with, suppose for simplicity that $h_s(\cdot)$ is differentiable everywhere. 
In the $k$-th iteration, a natural MD scheme that comes into mind for solving \eqref{eq:defn-reg-value-function}---namely, $\text{maximize}_{\pi}V_{\tau}^{\pi}(s_0)$ for a given initial state $s_0$---is the following update rule:
\begin{align}
	\pi^{(k+1)}(\cdot \mymid s) & = \arg\min_{ p \in \Delta(\cA)} \left\{  - \Big\langle \nabla _{\pi(\cdot | s)} V_\tau^{\pi}(s_0)\,\Big|_{\pi=\pi^{(k)}}, p \Big\rangle + \frac{\tau}{1-\gamma} d_{s_0}^{(k)}(s) h_s(p) + \frac{1}{\eta'} D_{h_s} \big( p, \pi^{(k)} (\cdot \mymid s) \big) \right\} \notag\\
	& = \arg\min_{ p \in \Delta(\cA)} \left\{ \frac{1}{1-\gamma}d_{s_0}^{(k)}(s) \Big\{ -  \big\langle  Q_\tau^{(k)}(s,\cdot), p \big\rangle + \tau h_s(p) \Big\} + \frac{1}{\eta'} D_{h_s} \big( p, \pi^{(k)} (\cdot \mymid s) \big) \right\} \notag\\
	& = \arg\min_{ p \in \Delta(\cA)} \left\{  -  \big\langle  Q_\tau^{(k)}(s,\cdot), p \big\rangle + \tau h_s(p)  + \frac{1}{\eta} D_{h_s}\big( p, \pi^{(k)} (\cdot \mymid s) \big) \right\}
	\label{eq:general-MD-first-idea}
\end{align}
for every state $s\in \cS$, which is a direct application of \eqref{eq:general-MD} to our setting. 
Here, we start with a learning rate $\eta'$, and obtain simplification by replacing $\eta'$ with ${\eta (1-\gamma)}/{ d_{s_0}^{(k)}(s)} $. 
Notably, the update strategy \eqref{eq:general-MD-first-idea} is invariant to the initial state $s_0$,  akin to natural policy gradient methods \citep{agarwal2019optimality}.

This update rule is well-defined for, say, the case when $h_s$ is the negative entropy, since the algorithm guarantees  $\pi^{(k)} > 0$ all the time and hence $h_s$ is always differentiable w.r.t.~the $k$-th iterate (see \citet{cccwc20}). 
In general, however, it is possible to encounter situations when the gradient of $h_s$ does not exist on the boundary (e.g., when $h_s$ represents a certain indicator function). 
To cope with such cases, we resort to a generalized version of Bregman divergence (e.g., \citet{kiwiel1997proximal,lan2011primal,lan2018optimal}). 
 To be  specific, we attempt to replace the usual Bregman divergence $D_{h_s}(p,q)$ by the following metric 
 \begin{align}
	 D_{h_s}(p,q; g_s) \coloneqq h_s(p) - h_s(q) - \langle g_s, p-q \rangle \geq 0,
	 \label{eq:defn-generalized-Bregman}
 \end{align}
 where $g_s$ can be any vector falling within the subdifferential $\partial h_s(q)$.  
 Here, the non-negativity condition in \eqref{eq:defn-generalized-Bregman} follows directly from the definition of the subgradient for any convex function.  
 The constraint on $g_s$ can be further relaxed by exploiting the requirement $p,q\in \Delta(\cA)$.  
In fact, for any vector $\xi_s = g_s- c_s 1$ (with $c_s\in \mathbb{R}$ some constant and $1$ the all-one vector), one can readily see that
 \begin{align}
	 D_{h_s}(p,q; g_s) &= h_s(p) - h_s(q) - \langle g_s, p-q \rangle = h_s(p) - h_s(q) - \langle \xi_s, p-q \rangle + c_s \langle 1, p-q \rangle \notag\\
	 & = h_s(p) - h_s(q) - \langle \xi_s, p-q \rangle = D_{h_s}(p,q; \xi_s),\label{eq:gbd1}
 \end{align}
where the last line is valid since $1^{\top} p = 1^{\top}q=1$. As a result, everything boils down to identifying a vector $\xi_s$ that falls within $\partial h_s(q)$ upon global shift.

 Towards this,  we propose the following iterative rule for designing such a sequence of vectors as surrogates for the subgradient of $h_s$:  
 \begin{subequations}
	 \label{eq:defn-xi-construction}
 \begin{align}
	 \xi^{(0)}(s,\cdot) &\in \partial h_s\big( \pi^{(0)}(\cdot \mymid s) \big); \\
	 \xi^{(k+1)}(s,\cdot) &= \frac{1}{1+\eta\tau}\xi^{(k)}(s,\cdot)+\frac{\eta}{1+\eta\tau}{{Q}^{(k)}_{\tau}(s,\cdot)}, \qquad k\geq 0, 
 \end{align}
 \end{subequations}
 where $\xi^{(k+1)}(s,\cdot) $ is updated as a convex combination of the previous $\xi^{(k)}(s,\cdot)$ and ${Q}^{(k)}_{\tau}(s,\cdot)$, where more emphasis is put on ${Q}^{(k)}_{\tau}(s,\cdot)$ when the learning rate $\eta$ is large.
 As asserted by the following lemma, the above vectors $\xi^{(k)}(s,\cdot)$ we construct satisfy the desired property, i.e., lying within the subdifferential of $h_s$ under suitable global shifts. 
 It is worth mentioning that these global shifts $\{c_s^{(k)}\}$ only serve as an aid to better understand the construction, but are not required during the algorithm updates.

 \begin{lemma}
	\label{lem:fact-xi-global-shift}
	 For all $k\geq 0$ and every $s\in \cS$, there exists a quantity $c_s^{(k)}\in \mathbb{R}$ such that
    \begin{equation}
        \xi^{(k)}(s,\cdot) - c_s^{(k)} 1  \in  \partial h_s\big( \pi^{(k)}(\cdot \mymid s) \big). 
        \label{eq:xi_in_subgrad}
    \end{equation}
	 In addition, for every $s\in \cS$, there exists a quantity $c_s^{\star}\in \mathbb{R}$ such that
	 \begin{equation}
		 \tau^{-1} Q_\tau^\star(s,\cdot) - c_s^\star 1 \in \partial h_s \big(\pi_\tau^\star(\cdot\mymid s) \big). 
		 \label{eq:xi_in_subgrad-star}
    	\end{equation}
 \end{lemma}
\begin{proof} See Appendix~\ref{sec:pf:lem:fact-xi-global-shift}. \end{proof}

Thus far, we have presented all crucial ingredients of our algorithm.  The whole procedure is summarized in Algorithm~\ref{alg:GPMD}, 
and will be referred to as {\em Generalized Policy Mirror Descent (GPMD)} throughout the paper. Interestingly, several well-known algorithms can be recovered as special cases of GPMD:
\begin{itemize}
	\item When the Bregman divergence $D_{h_s}(\cdot , \cdot)$ is taken as the KL divergence, GPMD reduces to the well-renowned NPG algorithm \citep{kakade2001natural} when $\tau=0$ (no regularization), and to the NPG algorithm with entropy regularization analyzed in \citep{cccwc20} when $h_s(\cdot)$ is taken as the negative Shannon entropy.
\item When $\eta=\infty$ (no divergence), GPMD reduces to regularized policy iteration in \cite{geist2019theory}; in particular, GPMD reduces to the standard policy iteration algorithm if in addition $\tau$ is also $0$.
\end{itemize}





\begin{algorithm}[th]
\caption{PMD with generalized Bregman divergence (GPMD)}
\label{alg:GPMD}
\textbf{Input:} initial policy iterate $\pi^{(0)}$, learning rate $\eta>0$.\\
\textbf{Initialize} $\xi^{(0)}$ so that $\xi^{(0)}(s,\cdot) \in \partial h_s\big( \pi^{(0)}(\cdot |s) \big)$ for all $s\in\mathcal{S}$.\\
\For{$k=0,1,\cdots,$}{

    For every $s \in \mathcal{S}$, set
   \begin{subequations}
    \begin{equation}
    \pi^{(k+1)}(\cdot | s)=\arg\min_{p\in\Delta(\mathcal{A})}\left\{- \big\langle Q_{\tau}^{(k)}(s,\cdot),p \big\rangle+\tau h_s(p)
	    +  \frac{1}{\eta}D_{h_s}\big(p, \pi^{(k)}(\cdot | s); \xi^{(k)}  \big)  \right\},
    \label{eq:GPMD_update}
    \end{equation}
    where 
    \begin{equation}
	    D_{h_s} \big(p,q ; \xi \big) \coloneqq h_s(p) - h_s(q) - \big\langle \xi(s,\cdot), p- q \big\rangle.
        \label{eq:strange_Bregman}
    \end{equation}\\
	For every $(s,a)\in \cS\times \cA$, compute
    \begin{equation}
	    \xi^{(k+1)}(s,a) = \frac{1}{1+\eta\tau}\xi^{(k)}(s,a)+\frac{\eta}{1+\eta\tau}{{Q}^{(k)}_{\tau}(s,a)}.
        \label{eq:xi_update}
    \end{equation}
   \end{subequations}
}

\end{algorithm}


\paragraph{Comparison with PMD \citep{lan21}.}
Before continuing, let us take a moment to point out the key differences between our algorithm GPMD and the PMD algorithm proposed in  \citet{lan21} in terms of algorithm designs. Although the primary exposition of PMD in \citet{lan21} fixes the Bregman divergence as the KL divergence, the algorithm also works in the presence of a generic Bregman divergence, whose relationship with the regularizer $h_s$ is, however, unspecified. Furthermore, GPMD adaptively sets this term to be the Bregman divergence generated by the regularizer $h_s$ in use, together with a carefully designed recursive update rule (cf.~\eqref{eq:defn-xi-construction}) to compute surrogates for the subgradient of $h_s$ to facilitate implementation. Encouragingly, this specific choice leads to a tailored performance analysis of GPMD, which was not present in and instead complementary with that of PMD \citep{lan21}. In truth, our theory offers linear convergence guarantees for more general scenarios by adapting to the geometry of the regularizer $h_s$; details to follow momentarily.


\section{Main results}
\label{sec:main-results}

This section presents our convergence guarantees for the GPMD method presented in Algorithm~\ref{alg:GPMD}. 
We shall start with the idealized case assuming that the update rule can be precisely implemented, and then discuss how to generalize it to the scenario with imperfect policy evaluation.

\subsection{Convergence of exact GPMD}
\label{sec:main-results-exact}

To start with, let us pin down the convergence behavior of GPMD, assuming that 
accurate evaluation of the policy $Q^{(k)}_{\tau}$ is available and the subproblem \eqref{eq:GPMD_update} can be solved perfectly.  
Here and below, we shall refer to the algorithm in this case as exact GPMD.  
Encouragingly, exact GPMD provably achieves global linear convergence from an arbitrary initialization, as asserted by the following theorem. 
\begin{theorem}[Exact GPMD]
\label{thm:linear_exact}
Suppose that Assumption~\ref{assumption:h-inf} holds. 
Consider any learning rate $\eta>0$, and set $\alpha:=\frac{1}{1+\eta\tau}$. Then the iterates of Algorithm~\ref{alg:GPMD} satisfy
\begin{subequations}
\begin{align}
	\big\Vert{Q}^{\star}_{\tau}-{Q}^{(k+1)}_{\tau} \big\Vert_{\infty} & \leq\gamma \big( 1 - (1-\alpha)(1-\gamma) \big)^{k}C_1, \\
	\big\Vert{V}^{\star}_{\tau}-{V}^{(k+1)}_{\tau} \big\Vert_{\infty}& \leq (\gamma + 2) \big( 1 - (1-\alpha)(1-\gamma) \big)^{k}C_1,
	\label{eq:V_exact_convergence}
\end{align}	
\end{subequations}
for all $k\geq 0$, 
where $C_1 \coloneqq \Vert{Q}^{\star}_{\tau}-{Q}^{(0)}_{\tau}\Vert_{\infty}+2\alpha\Vert{Q}^{\star}_{\tau}-\tau\xi^{(0)}\Vert_{\infty}$. 

In addition, if $h_s$ is $1$-strongly convex w.r.t.~the $\ell_1$ norm for some $s\in\cS$, then one further has
\begin{equation}
    \big\|\pi_\tau^\star(s) - \pi_\tau^{(k+1)}(s) \big\|_1 \le \tau^{-1}  \big( 1 - (1-\alpha)(1-\gamma) \big)^k C_1 ,
	\qquad k\geq 0.
\end{equation}
\end{theorem}
Our theorem confirms the fast global convergence of the GPMD algorithm, in terms of both the resulting regularized Q-value (if $h_s(\cdot)$ is convex) and the policy estimate (if $h_s(\cdot)$ is strongly convex). 
In summary, it takes GPMD no more than
\begin{subequations}
\label{eq:iteration-complexity-exact-PMD}
\begin{align}
\frac{1}{ (1-\alpha)(1-\gamma)  }\log\frac{C_{1}}{\varepsilon}=\frac{1+\eta\tau}{\eta\tau(1-\gamma)}\log\frac{C_{1}}{\varepsilon}
\end{align}
iterations to converge to an $\varepsilon$-optimal regularized Q-function (in the $\ell_{\infty}$ sense),  or
\begin{align}
	\frac{1}{ (1-\alpha)(1-\gamma)  }\log\frac{C_{1}}{\varepsilon\tau} = \frac{1+\eta\tau}{\eta\tau(1-\gamma)}\log\frac{C_{1}}{\varepsilon\tau}
\end{align}
\end{subequations}
iterations to yield an $\varepsilon$-approximation (w.r.t.~the $\ell_1$ norm error) of $\pi_{\tau}^{\star}$. 
The iteration complexity \eqref{eq:iteration-complexity-exact-PMD} is nearly dimension-free---namely, depending at most logarithmically on the dimension of the state-action space 
---making it scalable to large-dimensional problems.

\paragraph{Comparison with \citet[Theorems 1-3]{lan21}.}
To make clear our contributions, it is helpful to compare Theorem~\ref{thm:linear_exact} with the theory for the state-of-the-art algorithm PMD in \citet{lan21}.     
\begin{itemize}	
	\item {\em Linear convergence for convex regularizers under constant learning rates.} Suppose that constant learning rates are adopted for both GPMD and PMD. Our finding reveals that GPMD enjoys global linear convergence---in terms of both $\|{Q}^{\star}_{\tau}-{Q}^{(k+1)}_{\tau}\|_\infty$ and $\|{V}^{\star}_{\tau}-{V}^{(k+1)}_{\tau}\|_\infty$---even when the regularizer $h_s(\cdot)$ is only convex but not strongly convex. In contrast, \citet[{Theorem 2}]{lan21} provided only sublinear convergence guarantees (with an iteration complexity proportional to $1/\varepsilon$) for the case with convex regularizers, provided that constant learning rates are adopted.\footnote{In fact, \citet[Theorem 3]{lan21} suggests using a vanishing strongly convex regularization, as well as a corresponding increasing sequence of learning rates, in order to enable linear convergence for non-strongly-convex regularizers.}

	\item {\em A full range of learning rates.} Theorem~\ref{thm:linear_exact} reveals linear convergence of GPMD for a full range of learning rates, namely, our result is applicable to any $\eta>0$. In comparison, linear convergence was established in \citet{lan21} only when the learning rates are sufficiently large and when $h_s$ is $1$-strongly convex w.r.t.~the KL divergence. Consequently, the linear convergence results in \citet{lan21} do not extend to several widely used  regularizers such as negative Tsallis entropy and log-barrier functions (even after scaling), which are, in contrast, covered by our theory. It is worth noting that the case with small-to-medium learning rates is often more challenging to cope with in theory, given that its dynamics could differ drastically  from that of regularized policy iteration.    
		\item {\em Further comparison of rates under large learning rates.} \cite[Theorem 1]{lan21} achieves a contraction rate of $\gamma$ when the regularizer is strongly convex and the step size satisfies $\eta \ge \frac{1-\gamma}{\gamma \tau}$, while the contraction rate of GPMD is $1-\frac{\eta\tau}{1+\eta\tau}(1-\gamma)$ under the full range of the step size, which is slower but approaches the contraction rate $\gamma$ of PMD as $\eta$ goes to infinity. Therefore, in the limit $\eta \to\infty$, both GPMD and PMD achieve the contraction rate $\gamma$. As soon as $\eta \geq 1/\tau$, their iteration complexities are on the same order. 

\end{itemize}


\begin{remark}
While our primary focus is to solve the regularized RL problem, one might be tempted to apply GPMD as a means to solve unregularized RL;
	 for instance, one might run GPMD with the regularization parameter diminishing gradually in order to approach a  policy with the desired accuracy. We leave the details to Appendix~\ref{sec:adaptive}.
\end{remark}

\subsection{Convergence of approximate GPMD}
\label{sec:main-results-approx}

In reality, however, it is often the case that GPMD cannot be implemented in an exact manner, 
either because perfect policy evaluation is unavailable or because the subproblem \eqref{eq:GPMD_update} cannot be solved exactly.  
To accommodate these practical considerations, 
this subsection generalizes our previous result by permitting inexact policy evaluation and non-zero optimization error in solving \eqref{eq:GPMD_update}. 
The following assumptions make precise this imperfect scenario. 
\begin{assumption}[Policy evaluation error]
	\label{asp:estimation-error-inf}
	Suppose for any $k\geq 0$, we have access to an estimate $\hat{Q}^{(k)}_{\tau}$ obeying
	\begin{equation}
		\big\Vert\hat{Q}^{(k)}_{\tau}-Q^{(k)}_{\tau} \big\Vert_{\infty}\leq \epseval. 
	   \label{eq:estimation-error-inf}
   	\end{equation}
\end{assumption}
%

\begin{assumption}[Subproblem optimization error] 
\label{asp:optimization-error-approx}
Consider any policy $\pi$ and any vector $\xi \in \mathbb{R}^{|\cS||\cA|}$. Define 
\[
	f_{s} (p; \pi, \xi) \coloneqq - \big\langle Q(s,\cdot),p \big\rangle + \tau h_s(p) + \frac{1}{\eta}D_{h_s}\big( p, \pi(\cdot \mymid s); \xi(s,\cdot) \big),
\]
where $D_{h_s}(p,q;\xi)$ is defined in \eqref{eq:defn-generalized-Bregman}.  
Suppose there exists an oracle $G_{s,\epsopt}(Q,\pi,\xi)$, which is capable of returning $\pi'(\cdot \mymid s)$ such that
\begin{align}
\label{eq:oracle}
	f_{s} \big( \pi'(\cdot \mymid s) ; \pi, \xi \big) \leq  \min_{p\in\Delta(\mathcal{A})} f_{s} (p; \pi, \xi) + \epsopt. 
	%
\end{align}
%
%
%
\end{assumption}
Note that the oracle in Assumption~\ref{asp:optimization-error-approx} can be implemented efficiently in practice via various first-order methods \citep{beck2017first}. 
Under Assumptions~\ref{asp:estimation-error-inf} and \ref{asp:optimization-error-approx}, 
we can modify Algorithm~\ref{alg:GPMD} by replacing $\{{Q}^{(k)}_{\tau}\}$ with the estimate $\{\hat{Q}^{(k)}_{\tau}\}$, and invoking the oracle $G_{s,\epsopt}(Q,\pi,\xi)$ to solve the subproblem \eqref{eq:GPMD_update} approximately.  
The whole procedure, which we shall refer to as approximate GPMD, is summarized in Algorithm~\ref{alg:AGPMD}. 

\begin{algorithm}[h]
	\caption{Approximate PMD with generalized Bregman divergence (Approximate GPMD)}
	\label{alg:AGPMD}
	\textbf{Input:} initial policy $\pi^{(0)}$, learning rate $\eta>0$.\\
	\textbf{Initialize} $\hat{\xi}^{(0)}(s) \in \partial h_s\big(\pi^{(0)}(\cdot\mymid s) \big)$ for all $s\in\mathcal{S}$.\\
	\For{$k=0,1,\cdots,$}{
		For every $s \in \mathcal{S}$, invoke the oracle to obtain (cf.~\eqref{eq:oracle})
		%
		\begin{equation}
		\pi^{(k+1)}(s)=G_{s,\epsopt} \big(\hat{Q}^{(k)}_{\tau},\pi^{(k)},\hat{\xi}^{(k)} \big) .
		\label{eq:AGPMD_update}
		\end{equation}\\
		For every $(s,a)\in \cS\times \cA$, compute
		\begin{equation}
		\hat{\xi}^{(k+1)}(s,a) = \frac{1}{1+\eta\tau}\hat{\xi}^{(k)}(s,a)+\frac{\eta}{1+\eta\tau}{\hat{Q}^{(k)}_{\tau}(s,a)}.
		\label{eq:Axi_update}
		\end{equation}
	}
	
\end{algorithm}


The following theorem uncovers that approximate GPMD converges linearly---at the same rate as exact GPMD---before an error floor is hit.
\begin{theorem}[Approximate GPMD]
\label{thm:approximate-PMD}
Suppose that Assumptions~\ref{assumption:h-inf}, \ref{asp:estimation-error-inf} and \ref{asp:optimization-error-approx} hold. 
Consider any learning rate $\eta>0$. Then the iterates of Algorithm~\ref{alg:AGPMD} satisfy
\begin{subequations}
\begin{align}
	\Vert{Q}^{\star}_{\tau}-{Q}^{(k+1)}_{\tau}\Vert_{\infty}
	&\leq\gamma\left[\big( 1 - (1-\alpha)(1-\gamma) \big)^{k}C_1 
+ C_2\right],\label{eq:Q-converence-approx}\\
	\Vert{V}^{\star}_{\tau}-{V}^{(k+1)}_{\tau}\Vert_{\infty}
	& \leq (\gamma + 2)\left[\big( 1 - (1-\alpha)(1-\gamma) \big)^{k}C_1 
+ C_2\right] + (1-\alpha)\epsopt,
	\label{eq:V-converence-approx}
\end{align}
\end{subequations}
where $\alpha \coloneqq \frac{1}{1+\eta\tau}$, $C_1 $ is defined in Theorem~\ref{thm:linear_exact}, and 
\[
	C_2 \coloneqq \frac{1}{1-\gamma}\left[\left(2+\frac{2\gamma}{(1-\gamma)(1-\alpha)}\right)\epseval + \left(1 + \frac{2\gamma}{(1-\gamma)(1-\alpha)}\right)\epsopt\right].
\] 

In addition, if $h_s$ is $1$-strongly convex w.r.t.~the $\ell_1$ norm for any $s \in \mathcal{S}$, then we can further obtain 
\begin{subequations}	\label{eq:err_sc_bound1}
\begin{align}
		\Vert{Q}^{\star}_{\tau}-{Q}^{(k+1)}_{\tau}\Vert_{\infty}
		& \leq\gamma\left[\big( 1 - (1-\alpha)(1-\gamma) \big)^{k}C_1 
	+ C_3\right] , \\
	 \Vert{V}^{\star}_{\tau}-{V}^{(k+1)}_{\tau}\Vert_{\infty}
		&  \leq (\gamma + 2)\left[\big( 1 - (1-\alpha)(1-\gamma) \big)^{k}C_1 
+ C_3\right] + (1-\alpha)\epsopt , \\
	\big\| \pi_\tau^\star( \cdot\mymid s) - \pi^{(k+1)}(\cdot\mymid s) \big\|_1 
	& \leq\tau^{-1}\left[ \big( 1 - (1-\alpha)(1-\gamma) \big)^{k}C_1
+ C_3\right] + \sqrt{\frac{2\eta\epsopt}{1+\eta\tau}}, 
	\label{eq:pi-convergence-approx}
	\end{align}
\end{subequations}
%
where
\begin{equation}	\label{eq:err_sc_bound}
	C_3 \coloneqq \frac{1}{1-\gamma}\left[\left(2+\frac{\epseval\gamma}{\tau(1-\gamma)}\right)\epseval 
	+ \left(1 + \frac{4\gamma}{(1-\gamma)(1-\alpha)}\right)\epsopt\right] .
\end{equation}

\end{theorem}
%
%

	 In the special case where $\epsopt = 0$ and $\eta = \infty$, Algorithm~\ref{alg:AGPMD} reduces to regularized policy iteration, and the convergence result can be simplified as follows 
	 \[
		 \big\Vert{Q}^{\star}_{\tau}-{Q}^{(k)}_{\tau} \big\Vert_{\infty}
		 \leq\gamma^{k} \big\Vert{Q}^{\star}_{\tau}-{Q}^{(0)}_{\tau} \big\Vert_{\infty} + \frac{2\gamma\epseval}{(1-\gamma)^2}.
	 \]

In particular, when $h_s$ is taken as the negative entropy, our result strengthens the prior result established in \cite{cccwc20} for approximate entropy-regularized NPG method with $\epsopt=0$ over a wide range of learning rates. 
	Specifically, the error bound in \cite{cccwc20} reads $\gamma\cdot\frac{\epseval}{1-\gamma}\left(2+\frac{2\gamma}{\eta\tau}\right)$, where the second term in the bracket scales 
	inversely with respect to $\eta$ and therefore grows unboundedly as $\eta$ approaches $0$. In contrast,  \eqref{eq:err_sc_bound1} and \eqref{eq:err_sc_bound} suggest a bound $\gamma\cdot\frac{\epseval}{1-\gamma}\left(2+\frac{\epseval \gamma}{\tau(1-\gamma)}\right)$, which is independent of the learning rate $\eta$ in use and thus prevents the error bound from blowing up when the learning rate approaches $0$. Indeed, our result improves over the prior art \cite{cccwc20} whenever 
$\eta\leq\frac{2(1-\gamma)}{\epseval}$.

\begin{remark}[Sample complexities]
One might naturally ask how many samples are sufficient to learn an $\varepsilon$-optimal regularized Q-function, by leveraging sample-based policy evaluation algorithms in GPMD. Notice that it is straightforward to consider an expected version of Assumption \ref{asp:estimation-error-inf} as following:
	\[
		\begin{cases}
			\mathbb{E}\big[\big\Vert\hat{Q}^{(k)}_{\tau}-Q^{(k)}_{\tau} \big\Vert_{\infty}\big] &\leq \epseval;\\[1ex]
			\mathbb{E}\big[\big\Vert\hat{Q}^{(k)}_{\tau}-Q^{(k)}_{\tau} \big\Vert_{\infty}^2\big] &\leq \epseval^2,
		\end{cases}
	\]
where the expectation is with respect to the randomness in policy evaluation, then the convergence results in Theorem \ref{thm:approximate-PMD} apply to $\mathbb{E}\big[\Vert{Q}^{\star}_{\tau}-{Q}^{(k+1)}_{\tau}\Vert_{\infty}\big]$ and $\mathbb{E}\big[\big\| \pi_\tau^\star( \cdot\mymid s) - \pi_\tau^{(k+1)}(\cdot\mymid s) \big\|_1 \big]$ instead. This randomized version makes it immediately amenable to combine with, e.g., the rollout-based policy evaluators in \citet[Section 5.1]{lan21} to obtain (possibly crude) bounds on the sample complexity. We omit these straightforward developments.
\end{remark}

Roughly speaking, approximate GPMD is guaranteed to converge linearly to an error bound that scales linearly in both the policy evaluation error $\epseval$ and the optimization error $\epsopt$, thus confirming the stability of our algorithm vis-\`a-vis imperfect implementation of the algorithm. 
As before, our theory improves upon prior works by demonstrating linear convergence for a full range of learning rates even in the absence of strong convexity and smoothness.

\section{Analysis for exact GPMD (Theorem~\ref{thm:linear_exact})}
\label{sec:analysis-exact-gradient}


In this section, we present the analysis for our main result in Theorem~\ref{thm:linear_exact}, which follows a different framework from \cite{lan21}. 
Here and throughout, we shall often employ the following shorthand notation when it is clear from the context: 
\begin{align}
\label{eq:notation-simplified-s}
\begin{array}{lll}
	& \pi^{(k)}(s)\coloneqq\pi^{(k)}(\cdot\mymid s)\in\Delta(\cA),\qquad\qquad  & Q^{\pi}(s)\coloneqq Q^{\pi}(s,\cdot)\in\mathbb{R}^{|\mathcal{A}|},\\
	& \xi^{(k)}(s)\coloneqq\xi^{(k)}(s,\cdot)\in\mathbb{R}^{|\mathcal{A}|},\qquad & Q_{\tau}^{\pi}(s)\coloneqq Q_{\tau}^{\pi}(s,\cdot)\in\mathbb{R}^{|\mathcal{A}|},
\end{array}
\end{align}
in addition to those already defined in \eqref{eq:short_hand_v}.%

\subsection{Preparation: basic facts}

In this subsection, we single out a few basic results that underlie the proof of our main theorems.

\paragraph{Performance improvement.} 

To begin with, we demonstrate that GPMD enjoys a sort of monotonic improvements concerning the updates of both the value function and the Q-function, as stated in the following lemma. This lemma can be viewed as a generalization of the well-established policy improvement lemma in the analysis of NPG \citep{agarwal2019optimality,cccwc20} as well as PMD \citep{lan21}.
\begin{lemma}[Pointwise monotonicity]
	\label{lemma:perf_improve}
	 For any $(s,a)\in\mathcal{S} \times \mathcal{A}$ and any $k\geq0$, Algorithm~\ref{alg:GPMD} achieves 
	\begin{align}
		V^{(k+1)}_{\tau}(s)\geq V^{(k)}_{\tau}(s) \qquad \text{and} \qquad {Q}^{(k+1)}_{\tau}(s,a)\geq{Q}^{(k)}_{\tau}(s,a).
	\end{align}
\end{lemma}
\begin{proof}
	See Appendix~\ref{proof performance improvement}. 
\end{proof}
Interestingly, the above monotonicity holds simultaneously for all state-action pairs, and hence can be understood as a kind of pointwise monotonicity.


\paragraph{Generalized Bellman operator.} 

	Another key ingredient of our proof lies in the use of a generalized Bellman operator 
	$\mathcal{T}_{\tau,h}:\mathbb{R}^{\vert\mathcal{S}\vert\vert\mathcal{A}\vert}\rightarrow\mathbb{R}^{\vert\mathcal{S}\vert\vert\mathcal{A}\vert}$ associated with the regularizer $h=\{h_s\}_{s\in \cS}$. 
	Specifically, for any state-action pair $(s,a)$ and any vector $Q\in \mathbb{R}^{|\cS||\cA|}$, we define
\begin{align}
\label{eq:Bellman}
	\mathcal{T}_{\tau,h}(Q)(s,a)=r(s,a)+\gamma \mathop{\mathbb{E}}\limits_{s'\sim P(\cdot |s,a)}\left[\max_{p\in\Delta(\mathcal{A})} \Big\{ \big\langle Q(s'), p \big\rangle -\tau h_{s'}(p) \Big\} \right].
\end{align}
It is worth noting that 
this definition shares similarity with the regularized Bellman operator proposed in \citet{geist2019theory}, where the operator defined there is targeted at $V_{\tau}$, while ours is defined w.r.t.~$Q_{\tau}$.

The importance of this generalized Bellman operator is two-fold: it enjoys a desired contraction property, and its fixed point corresponds to the optimal regularized Q-function. 
These are generalizations of the properties for the classical Bellman operator, and are formally stated in the following lemma. The proof is deferred to Appendix \ref{sec:proof_Bellman}. 
\begin{lemma}[Properties of the generalized Bellman operator]
\label{lem:property-Bellman}
For any $\tau>0$, the operator $\mathcal{T}_{\tau,h}$ defined in \eqref{eq:Bellman} satisfies the following properties:
\begin{itemize}
	\item $\mathcal{T}_{\tau,h}$ is a contraction operator w.r.t.~the $\ell_\infty$ norm, namely, 
	for any $Q_1,Q_2\in\mathbb{R}^{\vert\mathcal{S}\vert\vert\mathcal{A}\vert}$, one has
	\begin{equation}
	\big\Vert\mathcal{T}_{\tau,h}(Q_1)-\mathcal{T}_{\tau,h}(Q_2) \big\Vert_{\infty}\leq\gamma\Vert Q_1-Q_2\Vert_{\infty}.
	\label{eq:Bellman_contract}
	\end{equation}

	\item The optimal regularized $Q$-function ${Q}^{\star}_{\tau}$ is a fixed point of $\mathcal{T}_{\tau,h}$, that is, 	
	\begin{equation}
		\mathcal{T}_{\tau,h}({Q}^{\star}_{\tau})={Q}^{\star}_{\tau}.
		\label{eq:Bellman_fixed_point}
	\end{equation}		
\end{itemize}
\label{lemma:Bellman}
\end{lemma}

\subsection{Proof of Theorem~\ref{thm:linear_exact}}

Inspired by \citet{cccwc20}, our proof consists of (i) characterizing the dynamics of $\ell_{\infty}$ errors and establishing a connection to a useful linear system with two variables, and (ii) analyzing the dynamics of this linear system directly. In what follows, we elaborate on each of these steps.

\paragraph{Step 1: error contraction and its connection to a linear system.}

With the assistance of the above preparations, we are ready to elucidate how to characterize the convergence behavior of $\Vert {Q}^{\star}_{\tau}-{Q}^{(k+1)}_{\tau}\Vert_{\infty}$.
Recalling the update rule of $\xi^{(k+1)}$ (cf.~\eqref{eq:xi_update}), we can deduce that
\[
	{Q}^{\star}_{\tau}-\tau\xi^{(k+1)}=\alpha \big( {Q}^{\star}_{\tau}-\tau\xi^{(k)} \big)+(1-\alpha)\big( {Q}^{\star}_{\tau}-{Q}^{(k)}_{\tau} \big)
\]
with $\alpha = \frac{1}{1+\eta\tau}$, thus indicating that
\begin{equation}
\label{eq:system1}
	\big\Vert {Q}^{\star}_{\tau}-\tau\xi^{(k+1)} \big\Vert_{\infty}
	\leq \alpha \big\Vert{Q}^{\star}_{\tau}-\tau\xi^{(k)} \big\Vert_{\infty} + (1-\alpha) \big\Vert{Q}^{\star}_{\tau}-{Q}^{(k)}_{\tau}\big\Vert_{\infty}.
\end{equation}

Interestingly, there exists an intimate connection between $\Vert{Q}^{\star}_{\tau}-{Q}^{(k+1)}_{\tau}\Vert_{\infty}$ and $\Vert{Q}^{\star}_{\tau}-\tau\xi^{(k+1)}\Vert_{\infty}$ 
that allows us to bound the former term by the latter. 
This is stated in the following lemma, with the proof postponed to Appendix~\ref{sec:pf_lemma_system2}.
\begin{lemma}
	\label{lemma:system2}
	Set $\alpha = \frac{1}{1+\eta\tau}$. The iterates of Algorithm \ref{alg:GPMD} satisfy
	\begin{equation}
		\label{eq:system2}
		\big\Vert{Q}^{\star}_{\tau}-{Q}^{(k+1)}_{\tau} \big\Vert_{\infty}
		\leq \gamma \big\Vert{Q}^{\star}_{\tau}-\tau\xi^{(k+1)} \big\Vert_{\infty}+\gamma\alpha^{k+1} \big\Vert {Q}^{(0)}_{\tau}-\tau\xi^{(0)} \big\Vert_{\infty}.
	\end{equation}		
\end{lemma}

The above inequalities \eqref{eq:system1} and \eqref{eq:system2}
can be succinctly described via a useful linear system with two variables $\Vert{Q}^{\star}_{\tau}-{Q}^{(k)}_{\tau}\Vert_{\infty}$ and $\Vert{Q}^{\star}_{\tau}-\tau\xi^{(k)}\Vert_{\infty}$, that is,
\begin{equation}
x_{k+1}\leq Ax_{k}+\gamma\alpha^{k+1}y,
\label{eq:system}
\end{equation}
where
\begin{equation}
\label{eq:system-specification-exact}
A \coloneqq
\begin{bmatrix}
\gamma(1-\alpha) & \gamma\alpha \\
1-\alpha & \alpha
\end{bmatrix}
, \qquad
x_k \coloneqq
\begin{bmatrix}
\Vert{Q}^{\star}_{\tau}-{Q}^{{(k)}}_{\tau}\Vert_{\infty} \\
\Vert{Q}^{\star}_{\tau}-\tau\xi^{(k)}\Vert_{\infty}
\end{bmatrix}
\qquad \text{and} \qquad
y \coloneqq
\begin{bmatrix}
\Vert{Q}^{(0)}_{\tau}-\tau\xi^{(0)}\Vert_{\infty} \\
0
\end{bmatrix}.
\end{equation}
This forms the basis for proving Theorem \ref{thm:linear_exact}.

\paragraph{Step 2: analyzing the dynamics of the linear system \eqref{eq:system}.}

Before proceeding, we note that a linear system similar to \eqref{eq:system} has been analyzed in \citet[Section 4.2.2]{cccwc20}. We intend to apply the following properties that have been derived therein:  
\begin{subequations}
	\label{eq:two-borrowed-properties-exact}
\begin{align}
	x_{k+1} & \leq A^{k+1}\left[x_0+\gamma(\alpha^{-1}A-I)^{-1}y\right],
	\label{eq:xk-recursive-formula-exact} \\
	\gamma(\alpha^{-1}A-I)^{-1}y & =
\begin{bmatrix}
0\\
\Vert{Q}^{(0)}_{\tau}-\tau\xi^{(0)}\Vert_{\infty}
\end{bmatrix}
,
\label{eq:non-negativity-alpha-A} \\
	A^{k+1} &=\big ((1-\alpha)\gamma+\alpha \big)^k \begin{bmatrix}
\gamma\\
1
\end{bmatrix}
\begin{bmatrix}
1-\alpha & \alpha
\end{bmatrix}
.
\label{eq:A-recursive-exact}
\end{align}
\end{subequations}
%
%
%
%
%
%
Substituting \eqref{eq:A-recursive-exact} and \eqref{eq:non-negativity-alpha-A} into (\ref{eq:xk-recursive-formula-exact})  and rearranging terms, we reach
\begin{align}
x_{k+1}&\leq \big((1-\alpha)\gamma+\alpha \big)^k\left((1-\alpha) \big\Vert{Q}^{\star}_{\tau}-{Q}^{(0)}_{\tau} \big\Vert_{\infty}+\alpha \big\Vert{Q}^{\star}_{\tau}-\tau\xi^{(0)} \big\Vert_{\infty}+\alpha \big\Vert{Q}^{(0)}_{\tau}-\tau\xi^{(0)} \big\Vert_{\infty}\right)
\begin{bmatrix}
\gamma\\
1
\end{bmatrix}\notag\\
&\leq \big((1-\alpha)\gamma+\alpha \big)^k\left(\big\Vert{Q}^{\star}_{\tau}-{Q}^{(0)}_{\tau}\big\Vert_{\infty}+2\alpha\big\Vert{Q}^{\star}_{\tau}-\tau\xi^{(0)}\big\Vert_{\infty}\right)
\begin{bmatrix}
\gamma\\
1
\end{bmatrix},
\end{align}
%
which taken together with the definition of $x_{k+1}$ gives
\begin{subequations}
\label{eq:Qstar-UB-123-exact}
\begin{align}
	\big\Vert{Q}^{\star}_{\tau}-{Q}^{(k+1)}_{\tau} \big\Vert_{\infty}
	& \leq\gamma \big((1-\alpha)\gamma+\alpha \big)^k\left( \big\Vert{Q}^{\star}_{\tau}-{Q}^{(0)}_{\tau} \big\Vert_{\infty}+2\alpha\big\Vert{Q}^{\star}_{\tau}-\tau\xi^{(0)}\big\Vert_{\infty}\right), \\
    \big\|Q_\tau^\star - \tau \xi^{(k+1)} \big\|_\infty 
	&\le \big((1-\alpha)\gamma+\alpha \big)^k\left( \big\Vert{Q}^{\star}_{\tau}-{Q}^{(0)}_{\tau} \big\Vert_{\infty}+2\alpha \big\Vert{Q}^{\star}_{\tau}-\tau\xi^{(0)} \big\Vert_{\infty}\right).
\end{align}
\end{subequations}

\paragraph{Step 3: controlling $\big\|  \pi_\tau^\star(s) - \pi^{(k+1)}(s) \big\|_1$ and $\big\Vert{V}^{\star}_{\tau}-{V}^{(k+1)}_{\tau} \big\Vert_{\infty}$.}
It remains to convert this result to an upper bound on $\big\|  \pi_{\tau}^\star(s) - \pi^{(k+1)}(s) \big\|_1$ and $\big\Vert{V}^{\star}_{\tau}-{V}^{(k+1)}_{\tau} \big\Vert_{\infty}$. 
By virtue of Lemma~\ref{lem:fact-xi-global-shift}, 
there exist two vectors $g_\tau^{\star}(s)\in \partial h_s\big( \pi_\tau^\star(s) \big)$, $g^{(k+1)}(s) \in \partial h_s\big(\pi ^{(k+1)}(s) \big)$ 
and two scalars $c_s^\star, c_s^{(k+1)}\in \mathbb{R}$ that satisfy
\[
\begin{cases}
	\tau^{-1} Q_\tau^\star (s) - c_s^\star 1 &= g_\tau^{\star}(s)\\
    \xi^{(k+1)}(s,\cdot) - c_s^{(k+1)} 1 &= g^{(k+1)}(s)
\end{cases}.
\]
It holds for all $s\in \mathcal{S}$ that
\begin{align}
	& V_\tau^\star (s) - V_\tau^{(k+1)}(s) \notag\\
	&= \big\langle Q_\tau^\star(s), \pi_\tau^\star(s) \big\rangle - \tau h_s(\pi_\tau^\star(s)) - \big\langle Q_\tau^{(k+1)}(s), \pi_\tau^{(k+1)}(s) \big\rangle + \tau h_s(\pi_\tau^{(k+1)}(s)) \notag\\
	&= \big\langle Q_\tau^\star(s) - Q_\tau^{(k+1)}(s), \pi_\tau^{(k+1)}(s) \big\rangle + \Big[\tau(h_s(\pi_\tau^{(k+1)}(s)) - h_s(\pi_\tau^\star(s))) - \big\langle Q_\tau^\star(s), \pi_\tau^{(k+1)}(s) - \pi_\tau^\star(s) \big\rangle\Big]\notag\\
	&\overset{\text{(i)}}{\le} \big\langle Q_\tau^\star(s) - Q_\tau^{(k+1)}(s), \pi_\tau^{(k+1)}(s) \big\rangle  + \big\langle \tau g^{(k+1)}(s) - Q_\tau^\star(s), \pi_\tau^{(k+1)}(s)-\pi_\tau^\star(s))\big\rangle\notag\\
	&= \big\langle Q_\tau^\star(s) - Q_\tau^{(k+1)}(s), \pi_\tau^{(k+1)}(s) \big\rangle  + \big\langle \tau \xi^{(k+1)}(s) - Q_\tau^\star(s), \pi_\tau^{(k+1)}(s)-\pi_\tau^\star(s))\big\rangle\notag\\
	&\le \big\| Q_\tau^\star(s) - Q_\tau^{(k+1)}(s) \big\|_\infty + 2 \big\|Q_\tau^\star(s) - \tau \xi^{(k+1)}(s)\big\|_\infty,
	\label{eq:V_bound_exact}
\end{align}
where (i) results from $h_s(\pi_\tau^{(k+1)}(s)) - h_s(\pi_\tau^\star(s)) \le \big\langle g^{(k+1)}(s), \pi_\tau^{(k+1)}(s)-\pi_\tau^\star(s) \big\rangle$.
Plugging \eqref{eq:Qstar-UB-123-exact} into \eqref{eq:V_bound_exact} completes the proof for \eqref{eq:V_exact_convergence}.

When $h_s$ is $1$-strongly convex w.r.t.~the $\ell_1$ norm, we can invoke the strong monotonicity property of a strongly convex function \citep[Theorem 5.24]{beck2017first} to obtain
\begin{align}
    \big\|  \pi_\tau^\star(s) - \pi^{(k+1)}(s) \big\|_1^2 
	&\le \big\langle \pi_\tau^\star(s) - \pi^{(k+1)}(s),  g_\tau^{\star}(s) - g^{(k+1)}(s)  \big\rangle\nonumber\\
    & = \big\langle \pi_\tau^\star(s) - \pi^{(k+1)}(s),  g_\tau^{\star}(s) + c_s^\star 1 - g^{(k+1)}(s) - c_s^{(k+1)} 1 \big\rangle\nonumber\\
    & \le \big\|\pi_\tau^\star(s) - \pi^{(k+1)}(s) \big\|_1 \big\|g_\tau^{\star}(s) + c_s^\star 1 - g^{(k+1)}(s) - c_s^{(k+1)} 1 \big\|_\infty \notag\\
	& =  \tau^{-1} \big\|\pi_\tau^\star(s) - \pi^{(k+1)}(s) \big\|_1 \big\| Q_\tau^\star (s) - \tau \xi^{(k+1)} (s) \big\|_{\infty} ,
    \label{eq:pi_conv_l1}
\end{align}
where the second line is valid since $\langle \pi_\tau^\star(s), 1\rangle = \langle \pi^{(k+1)}(s), 1\rangle=1$. 
This taken together with \eqref{eq:Qstar-UB-123-exact} gives rise to the advertised bound
\begin{align*}
    \big\|\pi_\tau^\star(s) - \pi^{(k+1)}(s) \big\|_1 
& \leq \tau^{-1} \big\| Q_\tau^\star (s) - \tau \xi^{(k+1)} (s) \big\|_{\infty} \\
	&\le \tau^{-1} \big( (1-\alpha)\gamma+\alpha \big)^k\left(\big \Vert{Q}^{\star}_{\tau}-{Q}^{(0)}_{\tau} \big\Vert_{\infty}+2\alpha \big\Vert{Q}^{\star}_{\tau}-\tau\xi^{(0)} \big\Vert_{\infty}\right).
\end{align*}
%

%



\section{Numerical experiments}
\label{sec:experiments}

In this section, we provide some simple numerical experiments to corroborate the effectiveness of the GPMD algorithm.

\subsection{Tsallis entropy}

While Shannon entropy is a popular choice of regularization, the discrepancy between the value function of the regularized MDP and the unregularized counterpart scales as $O(\frac{\tau}{1-\gamma} \log |\mathcal{A}|)$. In addition, the optimal policy under Shannon entropy regularization assigns positive mass to all actions and is hence non-sparse. To promote sparsity and obtain better control of the bias induced by regularization, \citet{lee2018sparse,lee2019tsallis} proposed to employ the Tsallis entropy \citep{tsallis1988possible} as an alternative. To be precise, for any vector $p \in \Delta(\mathcal{A})$, the associated Tsallis entropy is defined as
\[
	\mathsf{Tsallis}_q(p) = \frac{1}{q-1}\left(1 - \sum_{a\in \mathcal{A}} \big(p(a)\big)^q\right) 
   = \frac{1}{q-1}\mathbb{E}_{a\sim p}\left[1 - \big(p(a)\big)^{q-1} \right],
\]
where $q >0$ is often referred to as the entropic-index. When $q \to 1$, the Tsallis entropy reduces to the Shannon entropy.

We now evaluate numerically the performance of PMD and GPMD when applied to a randomly generated MDP with $|\mathcal{S}| = 200$ and $|\mathcal{A}| = 50$. Here, the transition probability kernel and the reward function are generated as follows. For each state-action pair $(s,a)$, we randomly select $20$ states to form a set $\mathcal{S}_{s,a}$, and set $P(s'|s, a) = 1/20$ if $s' \in \mathcal{S}_{s,a}$, and $0$ otherwise. The reward function is generated by $r(s, a) \sim  U_{s, a}\cdot U_{s}$, where $U_{s, a}$ and $ U_{s}$ are independent uniform random variables over $[0, 1]$. We shall set the regularizer as $h_s(p) = -\mathsf{Tsallis}_2(p)$ for all $s\in \cS$ with a regularization parameter $\tau = 0.001$. As can be seen from the numerical results displayed in Figure~\ref{fig:tsallis}(a),  GPMD enjoys a faster convergence rate compared to PMD.

\begin{figure}
    \centering
    \begin{tabular}{cc}
    \includegraphics[width=0.49\linewidth]{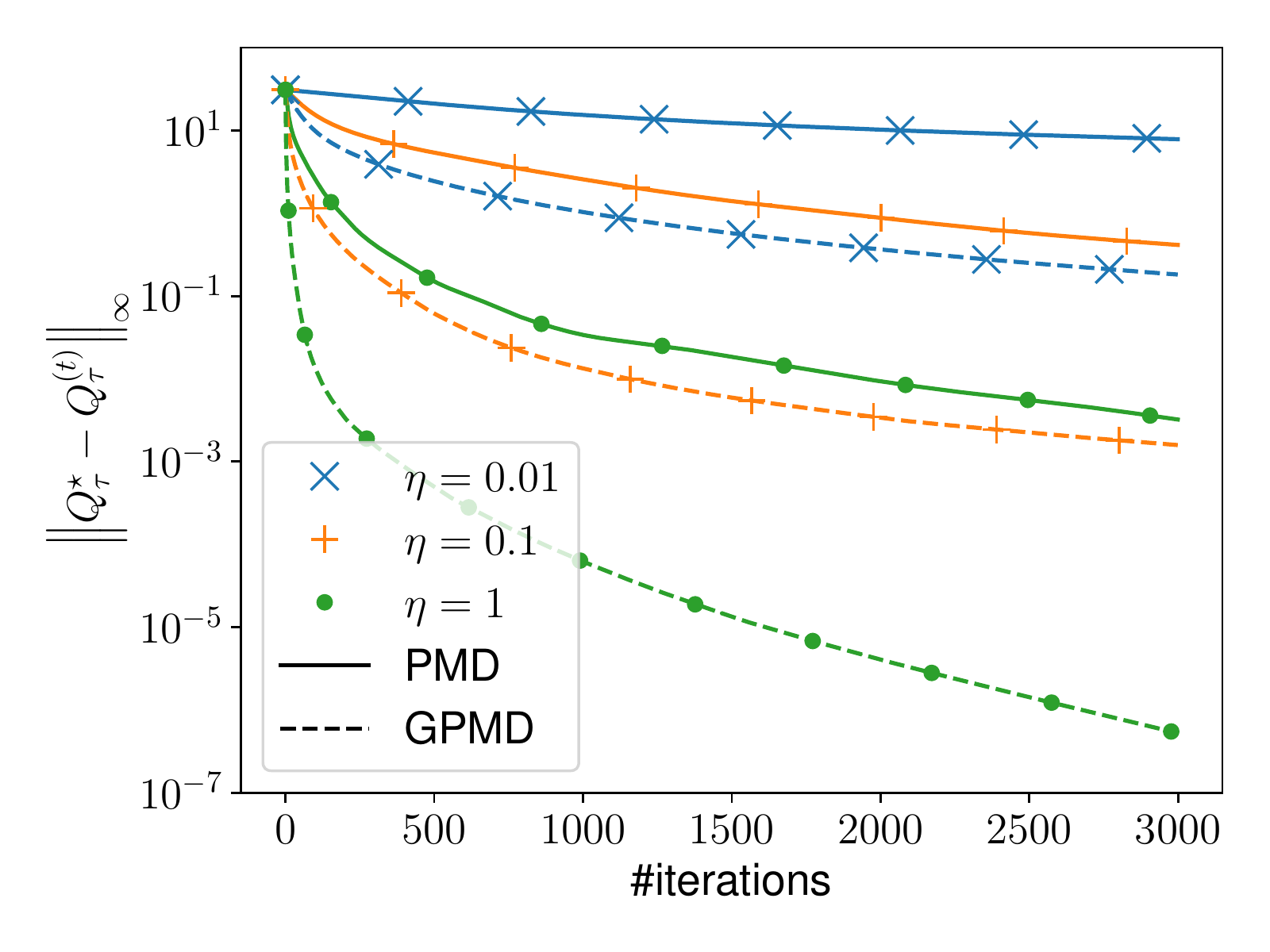} &
    \includegraphics[width=0.49\linewidth]{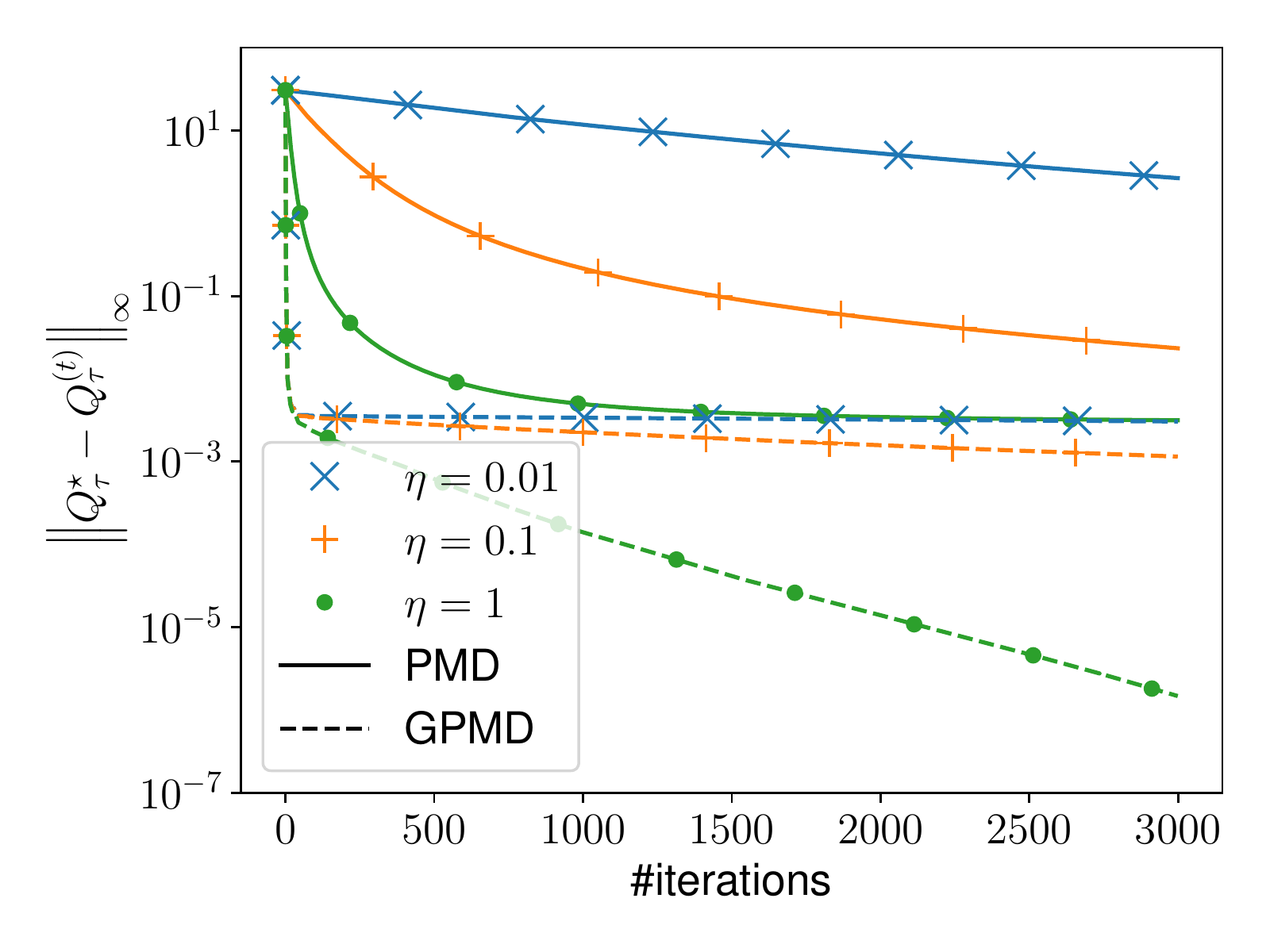} \\
    (a) Tsallis entropy regularization & (b) Log-barrier regularization
    \end{tabular}
	\caption{$\|Q_\tau^\star - Q_\tau^{(t)}\|_\infty$ versus the iteration count for both PMD and GPMD, for multiple choices of the learning rate $\eta$.
	The left plot (a) is concerned with Tsallis entropy regularization, whereas the right plot (b) concerns log-barrier regularization used in our constrained RL example. The error curves are averaged over 5 independent runs.}
    \label{fig:tsallis}
\end{figure}

\subsection{Constrained RL}

In reality, an agent with the sole aim of maximizing cumulative rewards might sometimes end up with unintended or even harmful behavior, due to, say, improper design of the reward function, or non-perfect simulation of physical laws. 
Therefore, it is sometimes necessary to enforce proper constraints on the policy
in order to prevent it from taking certain actions too frequently. 

To simulate this problem, we first solve a MDP with $|\mathcal{S}| = 200$ and $|\mathcal{A}| = 50$,  generated in the same way as in the previous subsection. We then pick 10 state-action pairs from the support of the optimal policy at random to form a set $\Psi$. We can ensure that $\pi_\tau^\star(a\mymid s) < \pi_{\rm max} = 0.1$ for all $(s,a) \in \Psi$ by adding the following log-barrier regularization with $\tau = 0.001$: 
\[
    h_s(p) = 
    \begin{cases}
	    \infty, & \text{if }(s, a) \in \Psi \text{ and } p(a) \ge \pi_{\rm max},\\
        -\log\big(\pi_{\rm max} - p(a) \big), & \text{if } (s, a) \in \Psi \text{ and } p(a) < \pi_{\rm max},\\
        0, & \text{otherwise}.
    \end{cases}
\] 
Numerical comparisons of PMD and GPMD when applied this problem are plotted in Figure \ref{fig:tsallis}(b). It is observed that PMD methods stall after reaching an error floor on the order of $10^{-2}$, while GPMD methods are able to converge to the optimal policy efficiently.


\section{Discussion}
\label{sec:discussion}

The present paper has introduced a generalized framework of policy optimization tailored to regularized RL problems.
We have proposed a generalized policy mirror descent (GPMD) algorithm that achieves dimension-free linear convergence,
which covers an entire range of learning rates and accommodates convex and possibly nonsmooth regularizers. 
Numerical experiments have been conducted to demonstrate the utility of the proposed GPMD algorithm. 
Our approach opens up a couple of future directions that are worthy of further exploration. 
For example, the current work restricts its attention to convex regularizers and tabular MDPs; 
it is of paramount interest to develop policy optimization algorithms when the regularizers are nonconvex and when sophisticated policy parameterization---including function approximation---is adopted.  
Understanding the sample complexities of  the proposed algorithm---when the policies are evaluated using samples collected over an online trajectory---is crucial in sample-constrained scenarios and is left for future investigation.   
Furthermore, it might be worthwhile to extend the proposed algorithm to accommodate multi-agent RL, with a representative example being regularized multi-agent Markov games \citep{cen2021fast,zhao2021provably,cen2022independent,cen2022faster}.

\section*{Acknowledgements}
S. Cen and Y.~Chi are supported in part by the grants ONR N00014-19-1-2404, NSF CCF-2106778, DMS-2134080, CCF-1901199, CCF-2007911, and CNS-2148212. S.~Cen is also gratefully supported by Wei Shen and Xuehong Zhang Presidential Fellowship, and Nicholas Minnici Dean's Graduate Fellowship in Electrical and Computer Engineering at Carnegie Mellon University. W. Zhan and Y.~Chen are supported in part by the Google Research Scholar Award, the Alfred P.~Sloan Research Fellowship, and the grants AFOSR FA9550-22-1-0198, ONR  N00014-22-1-2354, NSF CCF-2221009, CCF-1907661, IIS-2218713, and IIS-2218773. W. Zhan and J. Lee are supported in part by the ARO under MURI Award W911NF-11-1-0304, the Sloan Research Fellowship, NSF CCF 2002272, NSF IIS 2107304, and an ONR Young Investigator Award.

\appendix

\section{Proof of key lemmas}
\label{sec:proof-key-lemmas}

In this section, we collect the proof of several key lemmas. Here and  throughout, we use $\mathbb{E}_{\pi}[\cdot]$ to denote the expectation over the randomness of the MDP induced by policy $\pi$. 
We shall  follow the notation convention in \eqref{eq:notation-simplified-s} throughout. 
In addition, to further simplify notation, we shall abuse the notation by letting
\begin{subequations}
\label{eq:simplified-Bregman}
\begin{align}
	D_{h_s} (\widetilde{\pi}, \pi; \xi) &\coloneqq  D_{h_s} \big(\widetilde{\pi}(\cdot \mymid s), \pi(\cdot \mymid s); \xi(s,\cdot) \big) \\
	D_{h_s} (p, \pi; \xi) &\coloneqq  D_{h_s} \big(p, \pi(\cdot \mymid s); \xi(s,\cdot) \big) \\
	D_{h_s} (\pi, p; \xi) &\coloneqq  D_{h_s} \big( \pi(\cdot \mymid s), p; \xi(s,\cdot) \big)
\end{align}
\end{subequations}
for any policy $\pi$ and $\widetilde{\pi}$ and any $p\in \Delta(\cA)$, 
whenever it is clear from the context.

\subsection{Proof of Lemma~\ref{lem:fact-xi-global-shift}}
\label{sec:pf:lem:fact-xi-global-shift}

	We start by relaxing the probability simplex constraint (i.e., $p\in \Delta(\mathcal{A})$) in \eqref{eq:GPMD_update} with a simpler linear constraint $\sum_{a\in \mathcal{A}} p(a) = 1$ as follows
    \begin{equation}
        \begin{array}{ll}
		\text{minimize}_{p \in \mathbb{R}^{|\mathcal{A}|}}  & -\eta\big\langle Q^{(k)}_{\tau}(s),p \big\rangle+\eta\tau h_s(p)  +  D_{h_s} \big( p, \pi^{(k)};\xi^{(k)} \big)\\[1ex]
		\text{subject to} & \sum_{a\in \mathcal{A}} p(a) = 1 .		
        \end{array}
        \label{eq:GPMD_update_equiv}        
    \end{equation}
	To justify the validity of dropping the non-negative constraint, 
	we note that for any $p$ obeying $p(a) < 0$ for some $a \in \mathcal{A}$, our assumption on $h_s$ (see Assumption~\ref{assumption:h-inf}) leads to $h_s(p) = \infty$, which cannot possibly be the optimal solution. 
	This confirms the equivalence between \eqref{eq:GPMD_update} and \eqref{eq:GPMD_update_equiv}.

	Observe that the Lagrangian w.r.t.~\eqref{eq:GPMD_update_equiv} is given by
\begin{align*}
\mathcal{L}_{s} \big( p,\lambda_{s}^{(k)} \big) & =-\eta\big\langle Q_{\tau}^{(k)}(s),p\big\rangle+\eta\tau h_{s}(p)+h_{s}(p)-h_{s}\big(\pi^{(k)}(s)\big)-\big\langle p-\pi^{(k)}(s),\xi^{(k)}(s)\big\rangle+\lambda_{s}^{(k)}\left(\sum_{a\in\mathcal{A}}p(a)-1\right), 
\end{align*}
	where $\lambda_{s}^{(k)}\in \mathbb{R}$ denotes the Lagrange multiplier associated with the constraint $\sum_{a\in \mathcal{A}} p(a) = 1$. 
	Given that $\pi^{(k+1)}( s)$ is the solution to \eqref{eq:GPMD_update} and hence \eqref{eq:GPMD_update_equiv},  the optimality condition requires that
    \[
	    0\in \partial_{p}\mathcal{L}_{s}\big( p,\lambda_{s}^{(k)} \big) \,\Big|\,_{p=\pi^{(k+1)}( s)} 
	    = -\eta Q^{(k)}_{\tau}(s) + (1+\eta\tau) \partial h_s\big(\pi^{(k+1)}( s)\big) - \xi^{(k)}(s) + \lambda^{(k)}_s 1.
    \]
	Rearranging terms and making use of the construction \eqref{eq:defn-xi-construction}, we are left with
    \[
        \xi^{(k+1)}(s) - \frac{\lambda^{(k)}_s}{1+\eta\tau} 1=  \frac{1}{1+\eta\tau} \left[\eta Q^{(k)}_{\tau}(s) + \xi^{(k)}(s) - \lambda^{(k)}_s 1 \right] \in \partial h_s\big(\pi^{(k+1)}(s)\big),
    \]
	thus concluding the proof of the first claim \eqref{eq:xi_in_subgrad}.

	We now turn to the second claim \eqref{eq:xi_in_subgrad-star}.  In view of the property \eqref{eq:Bellman_fixed_point}, we have
	$$\pi_\tau^\star(s) = \arg\min_{p \in \Delta(\mathcal{A})} -\big\langle Q^{\star}_{\tau}(s),p \big\rangle+\tau h_s(p).$$ 
	This optimization problem is equivalent to 
	\begin{equation}
        \begin{array}{ll}
		\text{minimize}_{p \in \mathbb{R}^{|\mathcal{A}|}}  & -\big\langle Q^{\star}_{\tau}(s),p \big\rangle+\tau h_s(p) , \\[1ex]
		\text{subject to} & \sum_{a\in \mathcal{A}} p(a) = 1, 
        \end{array}
		\label{eq:GPMD_update_star_equiv}
    \end{equation}
	which can be verified by repeating a similar argument for \eqref{eq:GPMD_update_equiv}. The Lagrangian associated with \eqref{eq:GPMD_update_star_equiv} is
	\begin{align*}
		\mathcal{L}_{s} \big( p,\lambda_{s}^{\star} \big) & =-\big\langle Q_{\tau}^{\star}(s),p\big\rangle+\tau h_{s}(p)+\lambda_{s}^{\star}\left(\sum_{a\in\mathcal{A}}p(a)-1\right) ,
	\end{align*}
	where $\lambda_{s}^{\star} \in \mathbb{R} $ denotes the Lagrange multiplier. 
	Therefore, the first-order optimality condition requires that
	\[
	    0\in \partial_{p}\mathcal{L}_{s}\big( p,\lambda_{s}^{\star} \big) \,\Big|\,_{p=\pi^{\star}_\tau( s)} 
	    = - Q^{\star}_{\tau}(s) + \tau \partial h_s\big(\pi^{\star}_\tau( s)\big) + \lambda^{\star}_s 1,
    	\]
	which immediately finishes the proof.

\subsection{Proof of Lemma~\ref{lemma:perf_improve}}
\label{proof performance improvement}

We start by introducing the performance difference lemma that has previously been derived in \citet[Lemma 2]{lan21}. 
For the sake of self-containedness, we include a proof of this lemma in Appendix~\ref{sec:pf_perf_diff}.
\begin{lemma}[Performance difference]
\label{lemma:perf_diff}
For any two policies $\pi$ and $\pi'$, we have
\begin{equation}
V^{\pi'}_{\tau}(s)-V^{\pi}_{\tau}(s)
	=\frac{1}{1-\gamma} \mathop{\mathbb{E}}\limits_{s'\sim d^{\pi'}_s}\Big[ \big\langle Q^{\pi}_{\tau}(s'),\pi'(s')-\pi(s') \big\rangle-\tau h_{s'}\big(\pi'(s')\big)+\tau h_{s'}\big(\pi(s')\big)\Big],
\end{equation}	
where $d^{\pi}_s$ has been defined in \eqref{eq:defn-state-visitation}.
\end{lemma}

Armed with Lemma~\ref{lemma:perf_diff}, one can readily rewrite the difference $V^{(k+1)}_{\tau}(s)-V^{(k)}_{\tau}(s)$ between two consecutive iterates as follows
\begin{align}
\label{eq:perf_diff}
	&V^{(k+1)}_{\tau}(s)-V^{(k)}_{\tau}(s)\notag\\
	&\qquad =\frac{1}{1-\gamma} \mathop{\mathbb{E}}\limits_{s'\sim d^{(k+1)}_s} \Big[\big\langle Q^{(k)}_{\tau}(s'),\pi^{(k+1)}(s')-\pi^{(k)}(s') \big\rangle-\tau h_{s'}\big (\pi^{(k+1)}(s') \big)+\tau h_{s'}\big(\pi^{(k)}(s') \big) \Big].
\end{align}
It then comes down to studying the right-hand side of the relation (\ref{eq:perf_diff}), which can be accomplished via the following ``three-point'' lemma.
The proof of this lemma can be found in Appendix~\ref{sec:pf_update_three_point}.
\begin{lemma}
\label{lemma:update_three_point}
For any $s \in \mathcal{S}$ and any vector $p\in\Delta(\mathcal{A})$, we have
\begin{align*}
	& (1+\eta\tau)D_{h_s}\big( {p}, {\pi^{(k+1)} }; \xi^{(k+1)} \big) + D_{h_s}\big( {\pi^{(k+1)}}, {\pi^{(k)}}; \xi^{(k)}\big) - D_{h_s}\big( {p} , {\pi^{(k)}}; \xi^{(k)}\big) \\
	& \qquad = \eta\left[ \big\langle Q^{(k)}_{\tau}(s),\pi^{(k+1)}(s)-p \big\rangle+\tau h_s(p)-\tau h_s\big(\pi^{(k+1)}(s) \big)\right] .
\end{align*}
\end{lemma}
Taking $p=\pi^{(k)}(s)$ in Lemma~\ref{lemma:update_three_point} and combining it with \eqref{eq:perf_diff}, we arrive at
\begin{align*}
 & V_{\tau}^{(k+1)}(s)-V_{\tau}^{(k)}(s)\\
	& =\frac{1}{(1-\gamma)\eta} \mathop{\mathbb{E}}\limits_{s'\sim d_{s}^{(k+1)}}\left[(1+\eta\tau)D_{h_{s'}}\big(\pi^{(k)},\pi^{(k+1)};\xi^{(k+1)}\big)+D_{h_{s'}}\big(\pi^{(k+1)},\pi^{(k)};\xi^{(k)}\big)\right]\ge0 
\end{align*}
for any $s\in \cS$, thus establishing the advertised pointwise monotonicity w.r.t.~the regularized value function.

When it comes to the regularized Q-function, it is readily seen from the definition \eqref{eq:defn-reg-Q-function} that
\begin{align}
	{Q}^{(k+1)}_{\tau}(s,a) &=r(s,a)+\gamma \mathop{\mathbb{E}}\limits_{s'\sim P(\cdot |s,a)} \big[V^{(k+1)}_{\tau}(s') \big] \notag \\
	& \geq r(s,a)
	+ \gamma \mathop{\mathbb{E}}\limits_{s'\sim P(\cdot |s,a)} \big[V^{(k)}_{\tau}(s') \big]={Q}^{(k)}_{\tau}(s,a) \notag
\end{align}
for any $(s,a)\in \cS\times \cA$, where the last line is valid since $V_{\tau}^{(k+1)} \geq V_{\tau}^{(k)}$.
This concludes the proof.


\subsubsection{Proof of Lemma~\ref{lemma:perf_diff}}
\label{sec:pf_perf_diff}

For any two policies $\pi'$ and $\pi$, it follows from the definition \eqref{eq:defn-reg-value-function} of $V^{\pi}_{\tau}(s)$ that
\begin{align}
	& V^{\pi'}_{\tau}(s)-V^{\pi}_{\tau}(s)
	  =\mathbb{E}_{\pi'}\left[\sum_{t=0}^{\infty}\gamma^t\Big[r(s_t,a_t)-\tau h_{s_t}\big(\pi'(s_t) \big) \Big] \,\Big\vert\, s_0=s\right]-V^{\pi}_{\tau}(s) \notag\\
	&\qquad =\mathbb{E}_{\pi'}\left[\sum_{t=0}^{\infty}\gamma^t\Big[r(s_t,a_t)-\tau h_{s_t}\big(\pi'(s_t)\big)+V^{\pi}_{\tau}(s_t)-V^{\pi}_{\tau}(s_t)\Big] \,\Big\vert\, s_0=s\right]-V^{\pi}_{\tau}(s) \notag\\
	&\qquad =\mathbb{E}_{\pi'}\left[\sum_{t=0}^{\infty}\gamma^t\Big[r(s_t,a_t)-\tau h_{s_t}\big(\pi'(s_t)\big)+\gamma V^{\pi}_{\tau}(s_{t+1})-V^{\pi}_{\tau}(s_t)\Big] \,\Big\vert\, s_0=s\right]
	+ \mathbb{E}_{\pi'}\left[V^{\pi}_{\tau}(s_0) \,\Big\vert\, s_0=s\right]-V^{\pi}_{\tau}(s) \notag\\
	&\qquad =\mathbb{E}_{\pi'}\left[\sum_{t=0}^{\infty}\gamma^t\Big[r(s_t,a_t)-\tau h_{s_t}\big(\pi'(s_t)\big)+\gamma V^{\pi}_{\tau}(s_{t+1})-V^{\pi}_{\tau}(s_t)\Big] \,\Big\vert\, s_0=s\right] \notag\\
	&\qquad =\mathbb{E}_{\pi'}\left[\sum_{t=0}^{\infty}\gamma^t\Big[r(s_t,a_t)-\tau h_{s_t}\big(\pi(s_t)\big)+\gamma V^{\pi}_{\tau}(s_{t+1})-V^{\pi}_{\tau}(s_t)-\tau h_{s_t}\big(\pi'(s_t)\big)+\tau h_{s_t}\big(\pi(s_t)\big) \Big]
	\,\Big\vert\, s_0=s\right] \notag\\
	&\qquad =\mathbb{E}_{\pi'}\left[\sum_{t=0}^{\infty}\gamma^t\Big[Q^{\pi}_{\tau}(s_t,a_t) - \tau h_{s_t}\big(\pi(s_t)\big) -V^{\pi}_{\tau}(s_t)-\tau h_{s_t}\big( \pi'(s_t) \big)+\tau h_{s_t}\big (\pi(s_t) \big) \Big] \,\Big\vert\, s_0=s\right]\notag \\
	&\qquad =\frac{1}{1-\gamma} \mathop{\mathbb{E}}\limits_{s'\sim d^{\pi'}_s}\Big[ \big\langle Q^{\pi}_{\tau}(s'),\pi'(s')-\pi(s') \big\rangle-\tau h_{s'}\big(\pi'(s')\big)+\tau h_{s'}\big(\pi(s') \big)\Big] ,
	\label{eq:Vpiprime-Vpi-diff}
\end{align}
where the penultimate line comes from the definition \eqref{eq:defn-reg-Q-function}.
To see why the last line of \eqref{eq:Vpiprime-Vpi-diff} is valid, we make note of the following identity
\begin{align}
	& \mathop{\mathbb{E}}\limits_{a_t\sim\pi'(s_t)} \Big[ Q^{\pi}_{\tau}(s_t,a_t) - \tau h_{s_t}\big(\pi(s_t)\big) -V^{\pi}_{\tau}(s_t) \Big] \notag\\
    &\qquad =\mathop{\mathbb{E}}\limits_{a_t\sim\pi'(s_t)} \Big[ Q^{\pi}_{\tau}(s_t,a_t) - \tau h_{s_t}\big(\pi(s_t) \big) \Big] 
	- \mathop{\mathbb{E}}\limits_{a_t\sim\pi(s_t)} \Big[Q^{\pi}_{\tau}(s_t,a_t) - \tau h_{s_t}\big(\pi(s_t) \big) \Big] \notag\\
    &\qquad = \Big\langle Q^{\pi}_{\tau}(s_t) - \tau h_{s_t}\big(\pi(s_t)\big) \cdot 1,\pi'(s_t)-\pi(s_t) \Big\rangle \notag\\
    &\qquad = \Big\langle Q^{\pi}_{\tau}(s_t) ,\pi'(s_t)-\pi(s_t) \Big\rangle, \label{eq:Vpiprime-Vpi-diff-2}
\end{align}
where the first identity results from the relation \eqref{eq:reg-Q-V-relation}, and the last relation holds since $1^{\top}\pi'(s_t) = 1^{\top}\pi(s_t)=1$. 
The last line of \eqref{eq:Vpiprime-Vpi-diff} then follows immediately from the relation \eqref{eq:Vpiprime-Vpi-diff-2} and the definition \eqref{eq:defn-state-visitation} of $d_s^{\pi}$.

\subsubsection{Proof of Lemma~\ref{lemma:update_three_point}}
\label{sec:pf_update_three_point}

For any state $s \in \mathcal{S}$, we make the observation that
\begin{align*}
	& D_{h_{s}}\big(p,\pi^{(k)};\xi^{(k)}\big)  =h_{s}(p)-h_{s}\big(\pi^{(k)}(s)\big)-\big\langle p-\pi^{(k)}(s),\xi^{(k)}(s)\big\rangle\nonumber\\
 &  \qquad  =h_{s}(p)-h_{s}\big(\pi^{(k+1)}(s)\big)-\big\langle p-\pi^{(k+1)}(s),\xi^{(k)}(s)\big\rangle\nonumber\\
 &  \qquad\qquad+h_{s}\big(\pi^{(k+1)}(s)\big)-h_{s}\big(\pi^{(k)}(s)\big)-\big\langle\pi^{(k+1)}(s)-\pi^{(k)}(s),\xi^{(k)}(s)\big\rangle\nonumber\\
 &  \qquad=h_{s}(p)-h_{s}\big(\pi^{(k+1)}(s)\big)-\big\langle p-\pi^{(k+1)}(s),\xi^{(k+1)}(s)\big\rangle\nonumber\\
 &  \qquad\qquad+h_{s}\big(\pi^{(k+1)}(s)\big)-h_{s}\big(\pi^{(k)}(s)\big)-\big\langle\pi^{(k+1)}(s)-\pi^{(k)}(s),\xi^{(k)}(s)\big\rangle\nonumber\\
 &  \qquad\qquad+\big\langle p-\pi^{(k+1)}(s),\xi^{(k+1)}(s)-\xi^{(k)}(s)\big\rangle\big\rangle\nonumber\\
 &  \qquad =D_{h_{s}}\big(p,\pi^{(k+1)};\xi^{(k+1)}\big)+D_{h_{s}}\big(\pi^{(k+1)},\pi^{(k)};\xi^{(k)}\big)+\big\langle p-\pi^{(k+1)}(s),\xi^{(k+1)}(s)-\xi^{(k)}(s)\big\rangle\\
 &  \qquad =D_{h_{s}}\big(p,\pi^{(k+1)};\xi^{(k+1)}\big)+D_{h_{s}}\big(\pi^{(k+1)},\pi^{(k)};\xi^{(k)}\big)+\big\langle p-\pi^{(k+1)}(s),\eta Q_{\tau}^{(k)}(s)-\eta\tau\xi^{(k+1)}(s)\big\rangle,
\end{align*}
where the first and the fourth steps invoke the definition \eqref{eq:defn-generalized-Bregman} of the generalized Bregman divergence
and the last line results from the update rule \eqref{eq:xi_update}.
Rearranging terms, we are left with
\begin{align*}
	& \eta \big \langle Q^{(k)}_{\tau}(s),\pi^{(k+1)}(s)-p \big \rangle  \nonumber \\
	 & \qquad\qquad= \left\{ D_{h_s} \big( {p}, {\pi^{(k+1)}} ; \xi^{(k+1)}  \big) + D_{h_s} \big( {\pi^{(k+1)}}, {\pi^{(k)}}; \xi^{(k)}  \big) - D_{h_s}\big( {p} , {\pi^{(k)}}; \xi^{(k)}  \big)\right\}\\
    &\qquad \qquad \qquad+ \eta\tau \big\langle \xi^{(k+1)}(s),\pi^{(k+1)}(s)-p \big\rangle.
\end{align*}
Adding the term $\eta\tau \left\{ h_s(p) - h_s\big(\pi^{(k+1)}(s) \big)\right\}$ to both sides of this identity leads to
\begin{align*}
    &\eta\left[ \big\langle Q^{(k)}_{\tau}(s),\pi^{(k+1)}(s)-p \big\rangle+\tau h_s(p)-\tau h_s  \big(\pi^{(k+1)}(s) \big)\right]\\
	&\qquad = \left\{ D_{h_s}\big( {p}, {\pi^{(k+1)}}; \xi^{(k+1)} \big) + D_{h_s}\big( {\pi^{(k+1)}}, {\pi^{(k)}}; \xi^{(k)}  \big) 
	- D_{h_s}\big( {p}, {\pi^{(k)}}; \xi^{(k)}  \big) \right\}\\
    &\qquad \quad\quad+ \eta\tau \left(h_s(p) - h_s\big(\pi^{(k+1)}(s) \big) - \big\langle \xi^{(k+1)}(s),p-\pi^{(k+1)}(s) \big\rangle\right)\\
    &\qquad = (1+\eta\tau)D_{h_s}\big( {p}, {\pi^{(k+1)}  }; \xi^{(k+1)} \big) + D_{h_s}\big( {\pi^{(k+1)} }, \pi^{(k)} ; \xi^{(k)}  \big) - D_{h_s}\big( {p}, {\pi^{(k)} }; \xi^{(k)}  \big)
\end{align*}
as claimed, where the last line makes use of the definition \eqref{eq:defn-generalized-Bregman}.

\subsection{Proof of Lemma~\ref{lemma:Bellman}}
\label{sec:proof_Bellman}

In the sequel, we shall prove each claim in Lemma \ref{lemma:Bellman} separately.

\paragraph{Proof of the contraction property \eqref{eq:Bellman_contract}.}

For any $Q_1,Q_2\in \mathbb{R}^{|\cS||\cA|}$, the definition \eqref{eq:Bellman} of the generalized Bellman operator obeys
\begin{align*}
    	\mathcal{T}_{\tau,h}(Q_1)-\mathcal{T}_{\tau,h}(Q_2)
	&= \gamma \mathop{\mathbb{E}}\limits_{s'\sim P(\cdot |s,a)}\left[\max_{p\in\Delta(\mathcal{A})}\Big\{ \langle Q_1(s'), p\rangle -\tau h_{s'}(p) \Big\}
	-\max_{p\in\Delta(\mathcal{A})} \Big\{ \langle Q_2(s'), p\rangle-\tau h_{s'}(p) \Big\}\right]\\
	&\overset{\mathrm{(a)}}{\leq} \gamma\mathop{\mathbb{E}}\limits_{s'\sim P(\cdot |s,a)}\left[\max_{p\in\Delta(\mathcal{A})} \big\langle Q_1(s')-Q_2(s'), p \big\rangle \right]\\
	& \leq \gamma\mathop{\mathbb{E}}\limits_{s'\sim P(\cdot |s,a)}\left[\max_{p: \|p\|_1=1}   \Vert Q_1-Q_2\Vert_{\infty} \|p\|_1  \right] \\
	& = \gamma\Vert Q_1-Q_2\Vert_{\infty},
    \end{align*}
where (a) arises from the elementary fact $\max_xf(x)-\max_xg(x)\leq\max_x \big( f(x)-g(x) \big)$.


\paragraph{Proof of the fixed point property \eqref{eq:Bellman_fixed_point}.}

Towards this, let us first define 
\begin{align}
	\pi^{\dagger}(s) \coloneqq \arg\max_{p_{s}\in\Delta(\cA)} \mathop{\mathbb{E}}\limits _{a \sim p_{s}} \Big[{Q}^{\star}_{\tau}(s,a)-\tau h_{s}\big( p(s) \big) \Big] .
	\label{eq:defn-pi-dagger}
\end{align}
Then it can be easily verified that
\begin{align}
{Q}^{\star}_{\tau}(s,a)
	& =r(s,a)+\gamma\mathop{\mathbb{E}}\limits_{s_1\sim {P}(\cdot |s,a)}\left[\mathop{\mathbb{E}}\limits_{a_1\sim\pi^{\star}(s_1)}\Big[{Q}^{\star}_{\tau}(s_1,a_1)-\tau h_{s_1}\big(\pi^{\star}(s_1) \big)\Big]\right] \notag\\
	&\leq r(s,a)+\gamma\mathop{\mathbb{E}}\limits_{s_1\sim {P}(\cdot |s,a)}\left[\mathop{\mathbb{E}}\limits_{a_1\sim\pi^{\dagger}(s_1)}\Big[{Q}^{\star}_{\tau}(s_1,a_1)-\tau h_{s_1}\big( \pi^{\dagger}(s_1) \big)\Big]\right],
	\label{eq:Q-star-expansion-dagger}
\end{align}
where the first identity results from \eqref{eq:relation-reg-Q-reg-V}, and the second line arises from the maximizing property of $\pi^{\dagger}$ (see \eqref{eq:defn-pi-dagger}).

Note that the right-hand side of \eqref{eq:Q-star-expansion-dagger} involves the term ${Q}^{\star}_{\tau}(s_1,a_1)$,
which can be further upper bounded via the same argument for \eqref{eq:Q-star-expansion-dagger}. 
Successively repeating this upper bound argument (and the expansion) eventually allows one to obtain
\[
	{Q}^{\star}_{\tau}(s,a)\leq 
	r(s,a)+\gamma\mathbb{E}_{\pi^{\dagger}}\left[\sum_{t=1}^{\infty}\gamma^{t-1}\Big\{ r(s_t,a_t)-\tau h_{s_t}\big( \pi^{\dagger}(s_t) \big)\Big\} \,\Big\vert\, s_0=s, a_0=a\right]={Q}^{\pi^{\dagger}}_{\tau}(s,a).
\]
However, the fact that $\pi^{\star}$ is the optimal policy necessarily implies the following reverse inequality: 
$$
	{Q}^{\star}_{\tau}(s,a)\geq{Q}^{\pi^{\dagger}}_{\tau}(s,a).$$
Therefore, one must have 
\begin{align}	\label{eq:Qstar-Qdagger-equiv}
	{Q}^{\star}_{\tau}(s,a)={Q}^{\pi^{\dagger}}_{\tau}(s,a).
\end{align}

To finish up, it suffices to show that ${Q}^{\pi^{\dagger}}_{\tau}=\mathcal{T}_{\tau,h}({Q}^{\star}_{\tau})$.
To this end, it is observed that
\begin{align*}
{Q}^{\pi^{\dagger}}_{\tau}(s,a)
	&= r(s,a)+\gamma\mathop{\mathbb{E}}\limits_{s_1\sim {P}(\cdot |s,a)}\left[\mathop{\mathbb{E}}\limits_{a_1\sim\pi^{\dagger}(s_1)}\left[{Q}^{\pi^{\dagger}}_{\tau}(s_1,a_1)-\tau h_{s_1}\big( \pi^{\dagger}(s_1) \big)\right]\right]\\
	&\overset{\mathrm{(b)}}{=} r(s,a)+\gamma\mathop{\mathbb{E}}\limits_{s_1\sim {P}(\cdot |s,a)}\left[\mathop{\mathbb{E}}\limits_{a_1\sim\pi^{\dagger}(s_1)} \Big[ {Q}^{{\star}}_{\tau}(s_1,a_1)-\tau h_{s_1}\big(\pi^{\dagger}(s_1) \big)\Big]\right]\\
	&
\overset{\mathrm{(c)}}{=}r(s,a)+\gamma\mathop{\mathbb{E}}\limits _{s_{1}\sim{P}(\cdot|s,a)}\left[\max_{p\in\Delta(\mathcal{A})}\big\langle Q_{\tau}^{\star}(s_{1},a_{1}),p\big\rangle-\tau h_{s_{1}}(p)\Big]\right]
\\ 
	&= \mathcal{T}_{\tau,h}({Q}^{\star}_{\tau})(s,a),
\end{align*}
where (b) utilizes the fact \eqref{eq:Qstar-Qdagger-equiv}, (c) follows from the definition \eqref{eq:defn-pi-dagger} of $\pi^{\dagger}$, and the last identity is a consequence of the definition \eqref{eq:Bellman} of $\mathcal{T}_{\tau,h}$. The above results taken collectively demonstrate that ${Q}^{\star}_{\tau}=\mathcal{T}_{\tau,h}({Q}^{\star}_{\tau})$ as claimed.

\subsection{Proof of Lemma \ref{lemma:system2}}
\label{sec:pf_lemma_system2}
Recall that ${Q}^{(k+1)}_{\tau}={Q}^{\pi^{(k+1)}}_{\tau}$. In view of the relation \eqref{eq:relation-reg-Q-reg-V}, one obtains
\begin{align*}
Q_{\tau}^{(k+1)}(s,a) & =r(s,a)+\gamma\mathop{\mathbb{E}}\limits _{s'\sim P(\cdot|s,a)}\left[V_{\tau}^{(k+1)}(s')\right]\\
 & =r(s,a)+\gamma\mathop{\mathbb{E}}\limits _{s'\sim P(\cdot|s,a)}\left[\mathop{\mathbb{E}}\limits _{a'\sim\pi^{(k+1)}(s')}\left[Q_{\tau}^{(k+1)}(s',a')-\tau h_{s'}\big(\pi^{(k+1)}(s')\big)\right]\right]\\
 & =r(s,a)+\gamma\mathop{\mathbb{E}}\limits _{s'\sim P(\cdot|s,a)}\left[\big<Q_{\tau}^{(k+1)}(s'),\pi^{(k+1)}(s')\big>-\tau h_{s'}\big(\pi^{(k+1)}(s')\big)\right] .
\end{align*}
This combined with the fixed-point condition \eqref{eq:Bellman_fixed_point} allows us to derive
%
\begin{align}
\label{eq:pf_system2_step1}
	& {Q}^{\star}_{\tau}(s,a)-{Q}^{(k+1)}_{\tau}(s,a)
	=\mathcal{T}_{\tau,h}({Q}^{\star}_{\tau})(s,a)-\left\{ r(s,a)+\gamma \mathop{\mathbb{E}}\limits_{s'\sim {P}(\cdot| s,a)}\Big[ \big\langle {Q}^{(k+1)}_{\tau}(s'), \pi^{(k+1)}(s') \big\rangle - \tau h_{s'}\big(\pi^{(k+1)}(s') \big) \Big]\right\}\notag\\
	&\qquad =\mathcal{T}_{\tau,h}({Q}^{\star}_{\tau})(s,a)-\left\{r(s,a)+\gamma\mathop{\mathbb{E}}\limits_{s'\sim {P}(\cdot |s,a)} \Big[\big\langle \tau\xi^{(k+1)}(s'), \pi^{(k+1)}(s') \big\rangle - \tau h_{s'}\big(\pi^{(k+1)}(s') \big) \Big]\right\} \notag\\
	&\qquad\qquad\qquad -\gamma\mathop{\mathbb{E}}\limits_{s'\sim {P}(\cdot |s,a),a'\sim\pi^{(k+1)}(s')}\Big[{Q}^{(k+1)}_{\tau}(s',a')-\tau\xi^{(k+1)}(s',a') \Big] .
\end{align}


In what follows, we control each term on the right-hand side of \eqref{eq:pf_system2_step1} separately.

\paragraph{Step 1: bounding the 1st term on the right-hand side of \eqref{eq:pf_system2_step1}.}
		Lemma~\ref{lem:fact-xi-global-shift} tells us that $$\xi^{(k+1)}(s) - c_s^{(k+1)}  1 \in \partial h_s(\pi^{(k+1)}(s))$$ for some scalar $c_s^{(k+1)}\in \mathbb{R}$. 
		This important property allows one to derive
\begin{align}
	0 & \in - \xi^{(k+1)}(s) + c_{s}^{(k+1)}1 + \partial h_{s}\big(\pi^{(k+1)}(s)\big)= \partial {\mathcal{L}}_{k+1,s}\big(\pi^{(k+1)}(s); c_{s}^{(k+1)}\big)
	\label{eq:optimality-condition-xi-h-f}
\end{align}
where 
\[
	\mathcal{L}_{k+1,s}(p; \lambda) \coloneqq \underset{ \eqqcolon\, f_{k+1,s}(p) }{\underbrace{ - \big\langle\xi^{(k+1)}(s),p\big\rangle + h_{s}\big(p\big) }} + \lambda\, 1^{\top} p.
\]
Recognizing that the function $f_{k+1,s}(\cdot)$ is convex in $p$, 
we can view $\mathcal{L}_{k+1,s}(p; \lambda)$ as the Lagrangian of the following constrained convex problem with Lagrangian multiplier $\lambda\in \mathbb{R}$: 
\begin{align}
	\mathop{\text{minimize}}\limits_{p: 1^{\top}p=1} \quad f_{k+1,s}(p) = - \big\langle\xi^{(k+1)}(s),p\big\rangle + h_{s}\big(p\big).
	\label{eq:relaxed-constrained-fs}
\end{align}
The condition \eqref{eq:optimality-condition-xi-h-f} can then be interpreted as the optimality condition w.r.t.~the program \eqref{eq:relaxed-constrained-fs} and $\pi^{(k+1)}(s)$, meaning that
    \begin{align*}
	    f_{k+1,s}\big(\pi^{(k+1)}(s)\big) =  \min_{p: 1^{\top}p=1} f_{k+1,s}(p), 
    \end{align*}
    or equivalently,
    \begin{align}
	    \big\langle \xi^{(k+1)}(s), \pi^{(k+1)}(s) \big\rangle - h_s\big(\pi^{(k+1)}(s) \big) &= \max_{p : 1^{\top}p=1} \big\langle \xi^{(k+1)}(s), p \big\rangle - h_s(p).
	    \label{eq:optimal-pi-kplus1-relaxed}
    \end{align}

In addition, for any vector $p$ that does not obey $p\geq 0$, Assumption~\ref{assumption:h-inf} implies that $h_s(p)=\infty$, and hence $p$ cannot possibly be the optimal solution to 
$\max_{p\in \Delta(\cA)} \big\langle \xi^{(k+1)}(s), p \big\rangle - h_s(p)$. This together with \eqref{eq:optimal-pi-kplus1-relaxed} essentially implies that
    \begin{align}
	    \big\langle \xi^{(k+1)}(s), \pi^{(k+1)}(s) \big\rangle - h_s\big(\pi^{(k+1)}(s) \big) &= \max_{p\in \Delta(\cA)} \big\langle \xi^{(k+1)}(s), p \big\rangle - h_s(p).
	    \label{eq:optimal-pi-kplus1-non-relaxed}
    \end{align}
%
%
%
%
 As a consequence, we arrive at
    \begin{align}
        &\mathcal{T}_{\tau,h}({Q}^{\star}_{\tau})(s,a)
	    -\left\{ r(s,a)+\gamma\mathop{\mathbb{E}}\limits_{s'\sim P(\cdot |s,a)}\Big[ \big\langle \tau\xi^{(k+1)}(s'), \pi^{(k+1)}(s') \big\rangle - \tau h_{s'}\big(\pi^{(k+1)}(s')\big) \Big]\right\} \nonumber\\
	    &\qquad =\mathcal{T}_{\tau,h}({Q}^{\star}_{\tau})(s,a)-\left\{ r(s,a)+\gamma\mathop{\mathbb{E}}\limits_{s'\sim P(\cdot |s,a)}\Big[ \max_{p\in\Delta(\mathcal{A})} \Big\{ \big\langle \tau\xi^{(k+1)}(s'), p \big\rangle - \tau h_{s'}(p) \Big\} \Big]\right\} \nonumber\\
        &\qquad =\mathcal{T}_{\tau,h}({Q}^{\star}_{\tau})(s,a)-\mathcal{T}_{\tau,h}(\tau\xi^{(k+1)})(s,a)\nonumber\\
        &\qquad \le\gamma \big\Vert{Q}_\tau^{\star}-\tau\xi^{(k+1)} \big\Vert_{\infty},\label{eq:pf_system2_step2}
    \end{align}
    where the last step results from the contraction property \eqref{eq:Bellman_contract} in Lemma~\ref{lem:property-Bellman}.

\paragraph{Step 2: bounding the 2nd term on the right-hand side of \eqref{eq:pf_system2_step1}.}

    
Recall that $\alpha = \frac{1}{1+\eta\tau}$.	 Invoking the monotonicity property in Lemma~\ref{lemma:perf_improve} and the update rule \eqref{eq:xi_update}, we obtain
    \begin{align*}
    {Q}^{(k+1)}_{\tau}(s,a)-\tau\xi^{(k+1)}(s,a)
	    & =\alpha \Big\{ {Q}^{(k+1)}_{\tau}(s,a)-\tau\xi^{(k)}(s,a) \Big\} + (1-\alpha)\Big\{ {Q}^{(k+1)}_{\tau}(s,a)-{Q}^{(k)}_{\tau}(s,a)\Big\}\\
	    & \geq \alpha \Big\{ {Q}^{(k)}_{\tau}(s,a)-\tau\xi^{(k)}(s,a) \Big\}.
    \end{align*}
    Repeating this lower bound argument then yields
    \begin{align*}
	    {Q}^{(k+1)}_{\tau}(s,a)-\tau\xi^{(k+1)}(s,a)
	    & \geq \alpha^{k+1}\Big\{ {Q}^{(0)}_{\tau}(s,a)-\tau\xi^{(0)}(s,a) \Big\}\\
	    & \geq -\alpha^{k+1} \big\Vert{Q}^{(0)}_{\tau}-\tau\xi^{(0)} \big\Vert_{\infty},
    \end{align*}
    thus revealing that
    \begin{equation}
    -\mathop{\mathbb{E}}\limits_{s'\sim {P}(\cdot |s,a),a'\sim\pi^{k+1}(s')}\Big[ {Q}^{(k+1)}_{\tau}(s',a')-\tau\xi^{(k+1)}(s',a') \Big]
	    \leq \alpha^{k+1} \big\Vert {Q}^{(0)}_{\tau}-\tau\xi^{(0)} \big\Vert_{\infty}.
	    \label{eq:term-no2-UB}
    \end{equation}

\paragraph{Step 3: putting all this together.}

Substituting \eqref{eq:pf_system2_step2} and \eqref{eq:term-no2-UB} into \eqref{eq:pf_system2_step1} gives
\begin{align}
	0\leq {Q}^{\star}_{\tau}(s,a)-{Q}^{(k+1)}_{\tau}(s,a)
	\leq \gamma \big\Vert{Q}_\tau^{\star}-\tau\xi^{(k+1)} \big\Vert_{\infty} +  \alpha^{k+1} \big\Vert {Q}^{(0)}_{\tau}-\tau\xi^{(0)} \big\Vert_{\infty}
\end{align}
for all $(s,a)\in \cS\times \cA$, thus concluding the proof.

\section{Analysis for approximate GPMD (Theorem~\ref{thm:approximate-PMD})}
\label{sec:proof-linear-approximate}


The proof consists of three steps: (i) evaluating the performance difference between $\pi^{(k)}$ and $\pi^{(k+1)}$, (ii) establishing a linear system to characterize the error dynamic,
and (iii) analyzing this linear system to derive global convergence guarantees.
We shall describe the details of each step in the sequel. 
As before, we adopt the notational convention \eqref{eq:simplified-Bregman} whenever it is clear from the context.

\subsection{Step 1: bounding performance difference between consecutive iterates}
\label{sec:performance-diff-approx}

When only approximate policy evaluation is available, we are no longer guaranteed to have pointwise monotonicity as in the case of Lemma~\ref{lemma:perf_improve}. 
Fortunately, we are still able to establish an approximate versioin of  Lemma~\ref{lemma:perf_improve}, as stated below. 
%

\begin{lemma}[Performance improvement for approximate GPMD]
\label{lemma:aperf_improve}

For all $s\in\mathcal{S}$ and all $k\geq0$, we have
\[
V^{(k+1)}_{\tau}(s)\geq V^{(k)}_{\tau}(s)-\frac{1+\alpha}{1-\gamma}\epsopt - \frac{2}{1-\gamma}\epseval.
\]
In addition, if $h_s$ is 1-strongly convex w.r.t.~the $\ell_1$ norm for all $s \in \mathcal{S}$, then one further has
\[
    V^{(k+1)}_{\tau}(s)\ge V^{(k)}_{\tau}(s) -\frac{3+\alpha}{1-\gamma}\epsopt -\frac{\eta}{(2+\eta\tau)(1-\gamma)}\epseval^2.
\]
\end{lemma}
In words, while monotonicity is not guaranteed, this lemma precludes the possibility of $V^{(k+1)}_{\tau}(s)$ being be much smaller than $V^{(k)}_{\tau}(s)$, 
as long as both $\epseval$ and $\epsopt$ are reasonably small.

\subsubsection{Proof of Lemma~\ref{lemma:aperf_improve}}

\paragraph{The case when $h_s$ is convex.}

Let $\tilde{\pi}^{(k+1)}$ be the exact solution of the following problem
\begin{equation}
    \tilde{\pi}^{(k+1)}(s) =\arg\min_{p\in\Delta(\cA)}\left\{- \big\langle \hat{Q}^{(k)}_{\tau}(s),p \big\rangle 
	+ \tau h_s(p)+\frac{1}{\eta}D_{h_s}\big( p, \pi^{(k)}; \hat{\xi}^{(k)} \big)\right\}.
    \label{eq:exact_pi}
\end{equation}
%
%
With this auxiliary policy iterate $\tilde{\pi}^{(k+1)}$ in mind, we start by decomposing $V^{(k+1)}_{\tau}(s)-V^{(k)}_{\tau}(s)$ into the following three parts:
\begin{align}
&V^{(k+1)}_{\tau}(s)-V^{(k)}_{\tau}(s)\nonumber\\
&\qquad =\frac{1}{1-\gamma} \mathop{\mathbb{E}}\limits_{s'\sim d^{(k+1)}_s}\left[ \big\langle Q^{(k)}_{\tau}(s'),\pi^{(k+1)}(s')-\pi^{(k)}(s') \big\rangle-\tau h_{s'}\big(\pi^{(k+1)}(s')\big)+\tau h_{s'} \big(\pi^{(k)}(s') \big)\right]\nonumber\\
&\qquad =\frac{1}{1-\gamma}\mathop{\mathbb{E}}\limits_{s'\sim d^{(k+1)}_s}\left[ \big\langle \hat{Q}^{{(k)}}_{\tau}(s'),\tilde{\pi}^{(k+1)}(s')-\pi^{(k)}(s') \big\rangle-\tau h_{s'}\big(\tilde{\pi}^{(k+1)}(s')\big)
	+\tau h_{s'}\big(\pi^{(k)}(s')\big)\right]\nonumber\\
&\qquad \quad+\frac{1}{1-\gamma}\mathop{\mathbb{E}}\limits_{s'\sim d^{(k+1)}_s}\left[ \big\langle \hat{Q}^{{(k)}}_{\tau}(s'),\pi^{(k+1)}(s')-\tilde{\pi}^{(k+1)}(s') \big\rangle-\tau h_{s'}\big(\pi^{(k+1)}(s')\big)+\tau h_{s'}\big(\tilde{\pi}^{(k+1)}(s')\big)\right]\nonumber\\
&\qquad \quad+\frac{1}{1-\gamma}{\mathop{\mathbb{E}}\limits_{s'\sim d^{(k+1)}_s}\left[ \big\langle Q^{(k)}_{\tau}(s')-\hat{Q}^{{(k)}}_{\tau}(s'),\pi^{(k+1)}(s')-\pi^{(k)}(s') \big\rangle\right]} ,
\label{eq:aperf_improve_pf}
\end{align}
where the first identity arises from the performance difference lemma (cf.~Lemma \ref{lemma:perf_diff}). 
To continue, we seek to control each part of \eqref{eq:aperf_improve_pf} separately.
\begin{itemize}
	\item Regarding the first term of \eqref{eq:aperf_improve_pf},   replacing $\xi$ (resp.~$Q_\tau$) by $\hat{\xi}$ (resp.~$\hat{Q}_\tau$) in Lemma \ref{lemma:update_three_point} indicates that
    %
	%
    \begin{align}
 & \big\langle\hat{Q}_{\tau}^{(k)}(s'),\tilde{\pi}^{(k+1)}(s')-\pi^{(k)}(s')\big\rangle-\tau h_{s'}\big(\tilde{\pi}^{(k+1)}(s')\big)+\tau h_{s'}\big(\pi^{(k)}(s')\big)\nonumber\\
 & \qquad=\frac{1}{\eta}\left[(1+\eta\tau)D_{h_{s'}}\big(\pi^{(k)},\tilde{\pi}^{(k+1)}(s');\hat{\xi}^{(k+1)}\big)+D_{h_{s'}}\big(\tilde{\pi}^{(k+1)},\pi^{(k)};\hat{\xi}^{(k)}\big)\right]
	    \label{eq:aperf_improve_pf_1a}
\end{align}
    for all $s' \in \mathcal{S}$.

    \item As for the second term of \eqref{eq:aperf_improve_pf}, the definition of the oracle $G_{s, \epsopt}$ (see Assumption~\ref{asp:optimization-error-approx}) guarantees that 
    \begin{align}
	    &-\big\langle \hat{Q}_\tau^{(k)}(s'),\pi^{(k+1)}(s') \big\rangle  +  \tau h_{s'}\big(\pi^{(k+1)}(s')\big)+\frac{1}{\eta}D_{h_{s'}}\big( \pi^{(k+1)}, \pi^{(k)}; \hat{\xi}^{(k)} \big) \notag \\
        & \qquad \leq- \big\langle \hat{Q}_\tau^{(k)}(s'),\tilde{\pi}^{(k+1)}(s') \big\rangle
	    + \tau h_{s'}\big(\tilde{\pi}^{(k+1)}(s')\big)+\frac{1}{\eta}D _{h_{s'}}\big( \tilde{\pi}^{(k+1)}, \pi^{(k)}; \hat{\xi}^{(k)} \big)+\epsopt
	    \label{eq:opt_oracle}
    \end{align}
    for any $s' \in \mathcal{S}$. 
    Rearranging terms, we are left with
    \begin{align}
        &\big\langle \hat{Q}_\tau^{(k)}(s'),\pi^{(k+1)}(s') - \tilde{\pi}^{(k+1)}(s') \big\rangle-\tau h_{s'}\big(\pi^{(k+1)}(s')\big) + \tau h_{s'}\big(\tilde{\pi}^{(k+1)}(s') \big)\nonumber\\
	    &\qquad \ge - \frac{1}{\eta}D _{h_{s'}}\big( \tilde{\pi}^{(k+1)}, \pi^{(k)}; \hat{\xi}^{(k)}  \big) 
	    + \frac{1}{\eta}  D_{h_{s'}}\big( \pi^{(k+1)} , \pi^{(k)}; \hat{\xi}^{(k)}  \big)-\epsopt\nonumber\\
	    &\qquad =- \frac{1}{\eta}\left(D _{h_{s'}}\big( \tilde{\pi}^{(k+1)}, \pi^{(k)}; \hat{\xi}^{(k)}  \big)
	    - D_{h_{s'}} \big( \pi^{(k+1)}, \tilde{\pi}^{(k)}; \hat{\xi}^{(k)}  \big)\right) \notag\\
	    & \qquad\quad +\frac{1}{\eta}\left( D_{h_{s'}}\big( \pi^{(k+1)} , \pi^{(k)}; \hat{\xi}^{(k)}  \big)  
	    - D_{h_{s'}} \big( \pi^{(k+1)}, \tilde{\pi}^{(k)}; \hat{\xi}^{(k)}  \big) \right)-\epsopt.
	\label{eq:aperf_improve_pf_1b} 
    \end{align}
\end{itemize}
In addition, we note that the term
\[
    D_{h_{s'}}\big(\tilde{\pi}^{(k+1)},\pi^{(k)};\hat{\xi}^{(k)}\big)
\]
appears in both \eqref{eq:aperf_improve_pf_1a} and \eqref{eq:aperf_improve_pf_1b}, which can be canceled out when summing these two equalities. 
Specifically, adding \eqref{eq:aperf_improve_pf_1a} and \eqref{eq:aperf_improve_pf_1b} gives
\begin{align*}
    & \big\langle\hat{Q}_{\tau}^{(k)}(s'),\tilde{\pi}^{(k+1)}(s')-\pi^{(k)}(s')\big\rangle-\tau h_{s'}\big(\tilde{\pi}^{(k+1)}(s')\big)+\tau h_{s'}\big(\pi^{(k)}(s')\big)\nonumber\\
    &\quad +\big\langle \hat{Q}_\tau^{(k)}(s'),\pi^{(k+1)}(s') - \tilde{\pi}^{(k+1)}(s') \big\rangle-\tau h_{s'}\big(\pi^{(k+1)}(s')\big) + \tau h_{s'}\big(\tilde{\pi}^{(k+1)}(s') \big)\nonumber\\
    & \qquad\ge\frac{1}{\eta}\left[(1+\eta\tau)D_{h_{s'}}\big(\pi^{(k)},\tilde{\pi}^{(k+1)};\hat{\xi}^{(k+1)}\big)+D_{h_{s'}}\big({\pi}^{(k+1)},\tilde{\pi}^{(k)};\hat{\xi}^{(k)}\big)\right] \\
    & \qquad\quad +\frac{1}{\eta}\left( D_{h_{s'}}\big( \pi^{(k+1)} , \pi^{(k)}; \hat{\xi}^{(k)}  \big)  
    - D_{h_{s'}} \big( \pi^{(k+1)}, \tilde{\pi}^{(k)}; \hat{\xi}^{(k)} \big) \right)-\epsopt.
\end{align*}
Substituting this into \eqref{eq:aperf_improve_pf} and invoking the elementary inequality $|\langle a,b\rangle| \leq \|a\|_1\|b\|_{\infty}$ thus lead to 
\begin{align}
    &V^{(k+1)}_{\tau}(s)-V^{(k)}_{\tau}(s)\nonumber\\
	&\ge\frac{1}{1-\gamma}\mathop{\mathbb{E}}\limits_{s'\sim d^{(k+1)}_s}\left[\frac{1}{\eta}\left[(1+\eta\tau)D _{h_{s'}}\big( \pi^{(k)}, \tilde{\pi}^{(k+1)}; \hat{\xi}^{(k+1)}  \big) 
	+  D_{h_{s'}}\big( \pi^{(k+1)}, \tilde{\pi}^{(k)}; \hat{\xi}^{(k)}  \big)\right]\right]\nonumber\\
	&\quad + \frac{1}{1-\gamma}\mathop{\mathbb{E}}\limits_{s'\sim d^{(k+1)}_s}\left[\frac{1}{\eta}\left(D _{h_{s'}}\big( \pi^{(k+1)}, \pi^{(k)}; \hat{\xi}^{(k)}  \big) 
	- D_{h_{s'}}\big( \pi^{(k+1)}, \tilde{\pi}^{(k)}; \hat{\xi}^{(k)}  \big) \right)\right]\nonumber\\
    &\quad -\frac{\epsopt}{1-\gamma}- \frac{1}{1-\gamma}{\mathop{\mathbb{E}}\limits_{s'\sim d^{(k+1)}_s}\left[\big\|Q^{(k)}_{\tau}(s')-\hat{Q}^{{(k)}}_{\tau}(s')\big\|_\infty \big\|\pi^{(k+1)}(s')-\pi^{(k)}(s')\big\|_1\right]},
	\label{eq:aperf_improve_pf_2}
\end{align}
where the last line makes use of Assumption~\ref{asp:estimation-error-inf} and the fact $\|\pi^{(k+1)}(s)\|_1=\|\pi^{(k)}(s)\|_1=1$.

Following the discussion in Lemma~\ref{lem:fact-xi-global-shift}, we can see that $\hat{\xi}^{(k)}(s) - c_s^{(k)} 1 \in \partial h_s(\tilde{\pi}^{(k)}(s))$ with some constant $c_s^{(k)}$ for all $k$. 
This together with the convexity of $h_s$ (see \eqref{eq:defn-generalized-Bregman}) guarantees that
\begin{align}
	D_{h_s}\big( \pi^{(k)}, \tilde{\pi}^{(k+1)}; \hat{\xi}^{(k+1)}  \big)\ge0
	\qquad \text{and}\qquad 
	D_{h_s}\big( \pi^{(k+1)}, \tilde{\pi}^{(k)}; \hat{\xi}^{(k)}  \big)\ge0
	\label{eq:non-negative-Dhs-123}
\end{align}
for any $s\in \cS$, thus implying that the first term of \eqref{eq:aperf_improve_pf_2} is non-negative. 
It remains to control the second term in \eqref{eq:aperf_improve_pf_2}. Towards this, a little algebra gives
\begin{align}
	&D_{h_s}\big( \pi^{(k+1)}, \pi^{(k)}; \hat{\xi}^{(k)}  \big) 
	- D_{h_s}\big( \pi^{(k+1)}, \tilde{\pi}^{(k)}; \hat{\xi}^{(k)}  \big)  \nonumber\\
	&\qquad =-\left\{ h_s(\pi^{(k)}(s))-h_s\big(\tilde{\pi}^{(k)}(s)\big)- \big\langle \hat{\xi}^{(k)}(s), \pi^{(k)}(s)-\tilde{\pi}^{(k)}(s) \big\rangle\right\} \nonumber\\
&\qquad=-h_s\big(\pi^{(k)}(s)\big)+h_s\big(\tilde{\pi}^{(k)}(s)\big)+\left\langle \frac{1}{1+\eta\tau}\hat{\xi}^{(k-1)}(s) 
	+ \frac{\eta}{1+\eta\tau}\hat{Q}_\tau^{(k-1)}(s), \pi^{(k)}(s)-\tilde{\pi}^{(k)}(s)\right\rangle\nonumber\\
&\qquad=\frac{\eta}{1+\eta\tau}\left\{-\big\langle \hat{Q}^{(k-1)}_{\tau}(s),\tilde{\pi}^{(k)}(s) \big\rangle+\tau h_s\big(\tilde{\pi}^{(k)}(s)\big)
	+ \frac{1}{\eta}  D_{h_s}\big( \tilde{\pi}^{(k)} , \pi^{(k-1)}; \hat{\xi}^{(k-1)}  \big)\right.\nonumber\\
&\qquad\quad\quad\quad\ \ \left. - \left[- \big\langle \hat{Q}^{(k-1)}_{\tau}(s),\pi^{(k)}(s) \big\rangle+\tau h_s\big(\pi^{(k)}(s)\big)
	+  \frac{1}{\eta} D_{h_s}\big( \pi^{(k)}, \pi^{(k-1)}; \hat{\xi}^{(k-1)}  \big)\right]\right\}\nonumber\\
&\qquad\ge - \frac{\eta\epsopt}{1+\eta\tau}.
	\label{eq:aperf_improve_pf_2a}
\end{align}
Here, the first and the third lines follow from the definition \eqref{eq:defn-generalized-Bregman}, the second inequality comes from the construction \eqref{eq:Axi_update}, whereas the last step invokes the definition of the oracle \eqref{eq:oracle}. Substitution of \eqref{eq:non-negative-Dhs-123} and \eqref{eq:aperf_improve_pf_2a} into \eqref{eq:aperf_improve_pf_2} gives
%
\begin{align}
V^{(k+1)}_{\tau}(s)-V^{(k)}_{\tau}(s)
	& \geq 
	-\frac{1+\alpha}{1-\gamma}\epsopt
	- \frac{1}{1-\gamma}{\mathop{\mathbb{E}}\limits_{s'\sim d^{(k+1)}_s}\left[ \big\|Q^{(k)}_{\tau}(s')-\hat{Q}^{{(k)}}_{\tau}(s') \big\|_\infty \big\|\pi^{(k+1)}(s')-\pi^{(k)}(s') \big\|_1\right]}\label{eq:aperf_improve_pf_3}\\
	&\geq -\frac{1+\alpha}{1-\gamma}\epsopt-\frac{2}{1-\gamma}\epseval.
\nonumber
\end{align}


\paragraph{The case when $h_s$ is strongly convex.}

When $h_{s'}$ is $1$-strongly convex w.r.t.~the $\ell_1$ norm, the objective function of sub-problem \eqref{eq:exact_pi}
is $\frac{1+\eta\tau}{\eta}$-strongly convex w.r.t.~the $\ell_1$ norm. 
Taking this together with the $\epsopt$-approximation guarantee in Assumption~\ref{asp:optimization-error-approx}, we can demonstrate that
\begin{equation}
    \frac{1+\eta\tau}{2\eta} \big\|\tilde{\pi}^{(k+1)}(s') - \pi^{(k+1)}(s') \big\|_1^2 \le \epsopt \qquad
	\text{for all } k\ge 0 \text{ and } s'\in \mathcal{S}.
	\label{eq:pi_err_l1}
\end{equation}
Additionally, the strong convexity assumption also implies that
\begin{align*}
	& D_{h_{s'}}\big( \pi^{(k+1)}(s') -  \tilde{\pi}^{(k)}(s'); \xi^{(k)}(s') \big) 
  	\ge \frac{1}{2} \big\|\tilde{\pi}^{(k)}(s') - \pi^{(k+1)}(s') \big\|_1^2\\
    &\qquad = \frac{1}{2} \left(\big\|\tilde{\pi}^{(k)}(s') - \pi^{(k+1)}(s') \big\|_1^2 + \big\|\tilde{\pi}^{(k)}(s') - \pi^{(k)}(s') \big\|_1^2 \right) - \frac{1}{2} \big\|\tilde{\pi}^{(k)}(s') - \pi^{(k)}(s') \big\|_1^2\\
    &\qquad \ge \frac{1}{4} \left( \big\|\tilde{\pi}^{(k)}(s') - \pi^{(k+1)}(s') \big\|_1 + \big \|\tilde{\pi}^{(k)}(s') - \pi^{(k)}(s') \big\|_1 \right)^2 - \frac{1}{2} \big\|\tilde{\pi}^{(k)}(s') - \pi^{(k)}(s') \big\|_1^2\\
    &\qquad \ge \frac{1}{4} \big\|\pi^{(k)}(s') - \pi^{(k+1)}(s') \big\|_1^2 - \frac{\eta\epsopt}{1+\eta\tau},
\end{align*}
where the third line results from Young's inequality,
and the final step follows from \eqref{eq:pi_err_l1}. 
We can develop a similar lower bound on $D_{h_{s'}}\big( \pi^{(k)}, \tilde{\pi}^{(k+1)}; \xi^{(k+1)}\big)$ as well. Taken together, these lower bounds give
\begin{align*}
	&\frac{1}{\eta}\left[(1+\eta\tau)D_{h_{s'}}\big( \pi^{(k)}(s'), \tilde{\pi}^{(k+1)}(s'); \xi^{(k+1)} (s') \big) 
	+ D_{h_{s'}} \big( \pi^{(k+1)}(s'), \tilde{\pi}^{(k)}(s'); \xi^{(k)} (s') \big)\right]\\
        & \qquad \ge \frac{2+\eta\tau}{\eta} \left(\frac{1}{4} \big\| \pi^{(k)}(s') - \pi^{(k+1)}(s') \big\|_1^2 - \frac{\eta\epsopt}{1+\eta\tau} \right)\\
        & \qquad \ge \frac{2+\eta\tau}{4\eta} \big\|\pi^{(k)}(s') - \pi^{(k+1)}(s') \big\|_1^2 - 2\epsopt.
\end{align*}
In addition, it is easily seen that
\begin{align*}
    & - \big\|Q^{(k)}_{\tau}(s')-\hat{Q}^{{(k)}}_{\tau}(s') \big\|_\infty  \big\|\pi^{(k+1)}(s')-\pi^{(k)}(s') \big\|_1\\
    &\qquad \ge -\frac{1}{2}\left(\frac{2\eta}{2+\eta\tau} \big\|Q^{(k)}_{\tau}(s')-\hat{Q}^{{(k)}}_{\tau}(s') \big\|_\infty^2 + \frac{2+\eta\tau}{2\eta} \big\| \pi^{(k+1)}(s')-\pi^{(k)}(s') \big\|_1^2\right)\\
    &\qquad \ge -\frac{\eta}{2+\eta\tau}\epseval^2 - \frac{2+\eta\tau}{4\eta} \big\| \pi^{(k+1)}(s')-\pi^{(k)}(s') \big\|_1^2.
\end{align*}
Combining the above two inequalities with \eqref{eq:aperf_improve_pf_3}, we arrive at the advertised bound
\[
    V^{(k+1)}_{\tau}(s)-V^{(k)}_{\tau}(s) \ge -\frac{3+\alpha}{1-\gamma}\epsopt -\frac{\eta}{(2+\eta\tau)(1-\gamma)}\epseval^2.
\]
%

\subsection{Step 2: connecting the algorithm dynamic with a linear system}
\label{sec:linear-system-approx}


Now we are ready to discuss how to control $\big\Vert {Q}^{\star}_{\tau}-{Q}^{{(k)}}_{\tau} \big\Vert_{\infty}$.
In short, we intend to establish the connection among several intertwined quantities, and identify a simple linear system that captures the algorithm dynamic.

\paragraph{Bounding $\big\Vert {Q}^{\star}_{\tau}-\tau\hat{\xi}^{(k+1)} \big\Vert_{\infty}$.} From the definition of $\hat{\xi}^{(k+1)}$ in (\ref{eq:Axi_update}), we have
\begin{align}
\label{eq:asystem1}
\big\Vert {Q}^{\star}_{\tau}-\tau\hat{\xi}^{(k+1)} \big\Vert_{\infty} 
	&= \big\Vert \alpha({Q}^{\star}_{\tau}-\tau\hat{\xi}^{(k)})+(1-\alpha)({Q}^{\star}_{\tau}-{Q}^{{(k)}}_{\tau})+(1-\alpha)({Q}^{{(k)}}_{\tau}-\hat{{Q}}^{{(k)}}_{\tau}) \big\Vert_{\infty}\notag\\
	&\leq \alpha \big\Vert{Q}^{\star}_{\tau}-\tau\hat{\xi}^{(k)} \big\Vert_{\infty}+(1-\alpha)\big\Vert{Q}^{\star}_{\tau}-{Q}^{{(k)}}_{\tau} \big\Vert_{\infty}+(1-\alpha) \big\Vert{Q}^{{(k)}}_{\tau}-\hat{{Q}}^{{(k)}}_{\tau} \big\Vert_{\infty}\notag\\
	&\leq\alpha \big\Vert{Q}^{\star}_{\tau}-\tau\hat{\xi}^{(k)} \big\Vert_{\infty}+(1-\alpha)\big\Vert{Q}^{\star}_{\tau}-{Q}^{{(k)}}_{\tau} \big\Vert_{\infty}+(1-\alpha)\epseval,
\end{align}
where the last inequality is a consequence of Assumption~\ref{asp:estimation-error-inf}.

\paragraph{Bounding $-\min_{s,a}\big({Q}^{{(k+1)}}_{\tau}(s,a)-\tau\hat{\xi}^{(k+1)}(s,a)\big)$.} Applying the definition in (\ref{eq:Axi_update}) once again, we obtain
\begin{align}
-\big({Q}^{{(k+1)}}_{\tau}(s,a)-\tau\hat{\xi}^{(k+1)}(s,a)\big)
	& =- \alpha\big({Q}^{{(k)}}_{\tau}(s,a)-\tau\hat{\xi}^{(k)}(s,a)\big)+(1-\alpha)\big(\hat{{Q}}^{{(k)}}_{\tau}(s,a)-{Q}^{{(k)}}_{\tau}(s,a)\big)\notag\\
&\qquad +\big({Q}^{{(k)}}_{\tau}(s,a)-{Q}^{{(k+1)}}_{\tau}(s,a)\big)\notag\\
	&\leq-\alpha\big({Q}^{{(k)}}_{\tau}(s,a)-\tau\hat{\xi}^{(k)}(s,a)\big)+(1-\alpha+c_1)\epseval+c_2\epsopt,
	\label{med7}
\end{align}
where 
\begin{equation}
\begin{split}
c_{1}= & \begin{cases}
\frac{2\gamma}{1-\gamma},\qquad & \text{if }h_{s}\text{ is convex but not strongly convex},\\
\frac{\eta\epseval\gamma}{(2+\eta\tau)(1-\gamma)},\quad & \text{if }h_{s}\text{ is $1$-strongly convex w.r.t. the $\ell_1$ norm},
\end{cases}\\
c_{2} & =\begin{cases}
\frac{(\alpha+1)\gamma}{1-\gamma},\quad & \quad \text{if }h_{s}\text{ is convex but not strongly convex},\\
\frac{(\alpha+3)\gamma}{1-\gamma},\quad & \quad \text{if }h_{s}\text{ is $1$-strongly convex w.r.t. the $\ell_1$ norm}. 
\end{cases}
\label{eq:c_1c_2}
\end{split}
\end{equation}
Here, the last step of \eqref{med7} follows from Assumption~\ref{asp:estimation-error-inf} as well as the following relation: 
\[
	{Q}^{{(k)}}_{\tau}(s,a)-{Q}^{{(k+1)}}_{\tau}(s,a)
		=\gamma \mathop{\mathbb{E}}\limits_{s'\sim {P}(\cdot |s,a)} \left[V^{{(k)}}_{\tau}(s')-V^{{(k+1)}}_{\tau}(s')\right]\leq c_1\epseval+c_2\epsopt,
\]
where we have made use of Lemma \ref{lemma:aperf_improve}. 
%
Taking the maximum over $(s,a)$ on both sides of (\ref{med7}) yields
\begin{equation}
\label{eq:asystem2}
-\min_{s,a}\left({Q}^{{(k+1)}}_{\tau}(s,a)-\tau\hat{\xi}^{(k+1)}(s,a)\right)
	\leq-\alpha\min_{s,a} \big({Q}^{{(k)}}_{\tau}(s,a)-\tau\hat{\xi}^{(k)}(s,a) \big)+(1-\alpha+c_1)\epseval+c_2\epsopt.   
\end{equation}

\paragraph{Bounding $\big\Vert{Q}^{{\star}}_{\tau}-{Q}^{{(k+1)}}_{\tau}\big\Vert_{\infty}$.}

To begin with, let us decompose ${Q}^{{\star}}_{\tau}(s,a)-{Q}^{{(k+1)}}_{\tau}(s,a)$ into several parts. Invoking the relation \eqref{eq:Bellman_fixed_point} in Lemma~\ref{lem:property-Bellman} as well as the property \eqref{eq:relation-reg-Q-reg-V}, we reach 
%
\begin{align}
&{Q}^{{\star}}_{\tau}(s,a)-{Q}^{{(k+1)}}_{\tau}(s,a) \notag\\
&\qquad =\mathcal{T}_{\tau,h}({Q}^{\star}_{\tau})(s,a)-\left[r(s,a)+\gamma\mathop{\mathbb{E}}\limits_{s'\sim {P}(\cdot |s,a)}\Big[ \big\langle {Q}^{(k+1)}_{\tau}(s'), \pi^{(k+1)}(s') \big\rangle - \tau h_{s'}\big(\pi^{(k+1)}(s)\big) \Big]\right]\notag\\
&\qquad = \mathcal{T}_{\tau,h}({Q}^{\star}_{\tau})(s,a)-\left[r(s,a)+\gamma\mathop{\mathbb{E}}\limits_{s'\sim {P}(\cdot |s,a)}\Big[ \big\langle \tau\hat{\xi}^{(k+1)}(s'), \pi^{(k+1)}(s') \big\rangle - \tau h_{s'}\big(\pi^{(k+1)}(s)\big) \Big]\right]\notag\\
&\qquad\qquad -\gamma \mathop{\mathbb{E}}\limits_{s'\sim P(\cdot |s,a),a'\sim\pi^{(k+1)}(s')}\Big[{Q}^{(k+1)}_{\tau}(s',a')-\tau \widehat{\xi}^{(k+1)}(s',a') \Big]\notag\\
&\qquad =\left\{\mathcal{T}_{\tau,h}({Q}^{\star}_{\tau})(s,a)-\left[r(s,a)+\gamma\mathop{\mathbb{E}}\limits_{s'\sim {P}(\cdot |s,a)}\Big[ \big\langle \tau\hat{\xi}^{(k+1)}(s'), \tilde{\pi}^{(k+1)}(s') \big\rangle - \tau h_{s'}\big(\tilde{\pi}^{(k+1)}(s)\big) \Big]\right]\right\}\notag\\
&\qquad\quad -\tau\gamma\mathop{\mathbb{E}}\limits_{s'\sim {P}(\cdot |s,a)}\Big[ \big\langle\hat{\xi}^{(k+1)}(s'), {\pi}^{(k+1)}(s')-\tilde{\pi}^{(k+1)}(s') \big\rangle - h_{s'}\big({\pi}^{(k+1)}(s)\big) + h_{s'}\big(\tilde{\pi}^{(k+1)}(s) \big) \Big]\notag\\
	&\qquad\quad -\gamma\mathop{\mathbb{E}}\limits_{s'\sim {P}(\cdot |s,a),a'\sim\pi^{(k+1)}(s')}\Big[{Q}^{(k+1)}_{\tau}(s',a')-\tau \widehat{\xi}^{(k+1)}(s',a') \Big] . 
	\label{eq:decompose-12345}
\end{align}
In the sequel, we control the three terms in \eqref{eq:decompose-12345} separately.
\begin{itemize}
    \item To begin with, we repeat a similar argument as for \eqref{eq:pf_system2_step2} to show that
    \begin{align*}
	    &\mathcal{T}_{\tau,h}({Q}^{\star}_{\tau})(s,a)-\left[r(s,a)+\gamma \mathop{\mathbb{E}}\limits_{s'\sim P(\cdot |s,a)}\Big[ \big\langle \tau \widehat{\xi}^{(k+1)}(s'), \tilde{\pi}^{(k+1)}(s') \big\rangle - \tau h_{s'}\big(\tilde{\pi}^{(k+1)}(s) \big) \Big]\right]\\
        &\qquad =\mathcal{T}_{\tau,h}({Q}^{\star}_{\tau})(s,a)-\mathcal{T}_{\tau,h}(\hat{\xi}^{(k+1)})(s,a)\\
        &\qquad \le \gamma \big\Vert{Q}^{\star}_{\tau}-\tau\hat{\xi}^{(k+1)} \big\Vert_{\infty}.
    \end{align*}
    \item The second term of \eqref{eq:decompose-12345} can be bounded by applying \eqref{eq:aperf_improve_pf_2a} with $k$ replaced by $k+1$:
    \begin{align*}
        \big\langle\hat{\xi}^{(k+1)}(s'), {\pi}^{(k+1)}(s')-\tilde{\pi}^{(k+1)}(s') \big\rangle - h_{s'}\big({\pi}^{(k+1)}(s)\big) + h_{s'}\big(\tilde{\pi}^{(k+1)}(s)\big) \ge - \frac{\eta\epsopt}{1+\eta\tau}.
    \end{align*}
    \item As for the third term of \eqref{eq:decompose-12345}, taking the maximum over all $(s, a) \in \mathcal{S}\times\mathcal{A}$ gives
    \[
	    {Q}^{(k+1)}_{\tau}(s',a')-\tau \widehat{\xi}^{(k+1)}(s',a') \le -\min_{s,a}\left({Q}^{{(k+1)}}_{\tau}(s,a)-\tau\hat{\xi}^{(k+1)}(s,a)\right).
    \]
\end{itemize}

Taken together, the above bounds and the decomposition \eqref{eq:decompose-12345} lead to
\begin{equation}
\label{eq:asystem3}
 \big\Vert{Q}^{{*}}_{\tau}-{Q}^{{(k+1)}}_{\tau} \big\Vert_{\infty}
	\leq\gamma \big\Vert{Q}^{\star}_{\tau}-\tau\hat{\xi}^{(k+1)} \big\Vert_{\infty}-\gamma\min_{s,a}\left({Q}^{{(k+1)}}_{\tau}(s,a)-\tau\hat{\xi}^{(k+1)}(s,a)\right)+\gamma(1-\alpha)\epsopt.
\end{equation}

\paragraph{A linear system of interest.}
Combining (\ref{eq:asystem1}),(\ref{eq:asystem2}) and (\ref{eq:asystem3}), we reach the following linear system
\begin{equation}
\label{asystem}
z_{k+1}\leq Bz_{k}+b,
\end{equation}
where
\begin{equation*}
B \coloneqq
\begin{bmatrix}
\gamma(1-\alpha)&\gamma\alpha&\gamma\alpha\\
1-\alpha&\alpha&0\\
0&0&\alpha
\end{bmatrix}
,\qquad
z_k \coloneqq
\begin{bmatrix}
\big\Vert{Q}^{{\star}}_{\tau}-{Q}^{{(k)}}_{\tau} \big\Vert_{\infty}\\
\big\Vert {Q}^{\star}_{\tau}-\tau\hat{\xi}^{(k)} \big\Vert_{\infty}\\
-\min_{s,a}\big({Q}^{{(k)}}_{\tau}(s,a)-\tau\hat{\xi}^{(k)}(s,a)\big)
\end{bmatrix},
\end{equation*}
\begin{equation}
b \coloneqq
\begin{bmatrix}
\gamma(2-2\alpha+c_1)\epseval+\gamma(1-\alpha+c_2)\epsopt\\
(1-\alpha)\epseval\\
(1-\alpha+c_1)\epseval+c_2\epsopt
\end{bmatrix}.
	\label{eq:defn-b-approx}
\end{equation}
This linear system of three variables captures how the estimation error progresses as the iteration count increases.

\subsection{Step 3: linear system analysis}
\label{sec:linear-system-analysis-approx}

In this step, we analyze the behavior of the linear system \eqref{asystem} derived above. 
Observe that the eigenvalues and respective eigenvectors of the matrix $B$ are given by
\begin{equation}
\lambda_1=\alpha+(1-\alpha)\gamma, \qquad \lambda_2=\alpha,\qquad \lambda_3=0,
\end{equation}
\begin{equation}
v_1=
\begin{bmatrix}
\gamma\\
1\\
0
\end{bmatrix},
\qquad 
v_2=
\begin{bmatrix}
0\\
-1\\
1
\end{bmatrix},
\qquad
v_3=
\begin{bmatrix}
\alpha\\
\alpha-1\\
0
\end{bmatrix}.
\end{equation}
Armed with these, we can decompose $z_0$ in terms of the eigenvectors of $B$ as follows
\begin{align}
\label{eq:z0_decomp}
z_0&\leq
\begin{bmatrix}
\Vert{Q}^{{\star}}_{\tau}-{Q}^{{(0)}}_{\tau}\Vert_{\infty}\\
\Vert {Q}^{\star}_{\tau}-\tau\hat{\xi}^{(0)}\Vert_{\infty}\\
\Vert {Q}^{(0)}_{\tau}-\tau\hat{\xi}^{(0)}\Vert_{\infty}
\end{bmatrix}\notag\\
&=\frac{1}{\alpha+(1-\alpha)\gamma}\left[(1-\alpha) \big\Vert{Q}^{{\star}}_{\tau}-{Q}^{{(0)}}_{\tau} \big\Vert_{\infty}+\alpha \big\Vert {Q}^{\star}_{\tau}-\tau\hat{\xi}^{(0)} \big\Vert_{\infty}+\alpha \big\Vert {Q}^{(0)}_{\tau}-\tau\hat{\xi}^{(0)} \big\Vert_{\infty}\right]v_1\notag\\
&\qquad + \big\Vert {Q}^{\star}_{\tau}-\tau\hat{\xi}^{(0)} \big\Vert_{\infty}v_2+e_z v_3\notag\\
&\leq  \frac{1}{\alpha+(1-\alpha)\gamma}\left[\big\Vert{Q}^{{\star}}_{\tau}-{Q}^{{(0)}}_{\tau} \big\Vert_{\infty}+2\alpha \big\Vert {Q}^{\star}_{\tau}-\tau\hat{\xi}^{(0)} \big\Vert_{\infty}\right]v_1
+\big\Vert {Q}^{\star}_{\tau}-\tau\hat{\xi}^{(0)} \big\Vert_{\infty}v_2+e_z v_3,
\end{align}
where $e_z \in \mathbb{R}$ is some constant that does not affect our final result. Also, the vector $b$ defined in \eqref{eq:defn-b-approx} satisfies
\begin{align}
    b&\le 
    \begin{bmatrix}
        \gamma(2-2\alpha+c_1)\epseval+\gamma(1-\alpha+c_2)\epsopt\\
        (1-\alpha)\epseval + (1-\alpha)\epsopt\\
        (1-\alpha+c_1)\epseval+c_2\epsopt
    \end{bmatrix}\nonumber\\
    &= \big[(2-2\alpha + c_1)\epseval + (1-\alpha+c_2)\epsopt\big] v_1 + \big[(1-\alpha + c_1)\epseval+c_2\epsopt \big] v_2.
    \label{eq:b_decomp}
\end{align}
%

 Using the decomposition in (\ref{eq:z0_decomp}) and (\ref{eq:b_decomp}) and applying the system relation~(\ref{asystem}) recursively, we can derive
\begin{align*}
z_{k+1}&\le B^{k+1}z_0 + \sum_{t=0}^k B^{k-t}b\\
&\le B^{k+1} \left[\frac{1}{\alpha+(1-\alpha)\gamma}\left[\big\Vert{Q}^{{\star}}_{\tau}-{Q}^{{(0)}}_{\tau} \big\Vert_{\infty}+2\alpha \big\Vert {Q}^{\star}_{\tau}-\tau\hat{\xi}^{(0)} \big\Vert_{\infty}\right]v_1+ \big\Vert {Q}^{\star}_{\tau}-\tau\hat{\xi}^{(0)} \big\Vert_{\infty}v_2+e_z v_3\right]\\
&\quad + \sum_{t=0}^{k} B^{k-t}\Big[\big[(2-2\alpha + c_1)\epseval + (1-\alpha+c_2)\epsopt\big] v_1 + \big[(1-\alpha + c_1)\epseval+c_2\epsopt\big] v_2\Big]\\
&=\left[\lambda_1^k\left( \big\Vert{Q}^{\star}_{\tau}-{Q}^{(0)}_{\tau} \big\Vert_{\infty}+2\alpha \big\Vert {Q}^{\star}_{\tau}-\tau\hat{\xi}^{(0)} \big\Vert_{\infty}\right)+\frac{1-\lambda_1^{k+1}}{1-\lambda_1}\Big[(2-2\alpha + c_1)\epseval + (1-\alpha+c_2)\epsopt\Big]\right]v_1\\
&\quad+\Big[\lambda_2^{k+1} \big\Vert {Q}^{\star}_{\tau}-\tau\hat{\xi}^{(0)} \big\Vert_{\infty}+\frac{1-\lambda_2^{k+1}}{1-\lambda_2}[(1-\alpha + c_1)\epseval+c_2\epsopt]\Big]v_2.
\end{align*}
%
Recognizing that the first two entries of $v_2$ are non-positive, we can discard the term involving $v_2$ and obtain
\begin{align*}
&\begin{bmatrix}
    \Vert{Q}^{\star}_{\tau}-{Q}^{{(k+1)}}_{\tau}\Vert_{\infty}\\
    \|Q_\tau^\star - \tau \hat{\xi}^{(k+1)}\|_\infty
\end{bmatrix}\\
&\le \left[\lambda_1^k\left( \big\Vert{Q}^{\star}_{\tau}-{Q}^{(0)}_{\tau} \big\Vert_{\infty}+2\alpha \big\Vert {Q}^{\star}_{\tau}-\tau\hat{\xi}^{(0)} \big\Vert_{\infty}\right)+\frac{1-\lambda_1^{k}}{1-\lambda_1}\big[(2-2\alpha + c_1)\epseval + (1-\alpha+c_2)\epsopt\big]\right]
\begin{bmatrix}
    \gamma\\
    1
\end{bmatrix}\\
&\le \bigg[\lambda_1^k\left(\big\Vert{Q}^{\star}_{\tau}-{Q}^{(0)}_{\tau}\big\Vert_{\infty}+2\alpha\big\Vert {Q}^{\star}_{\tau}-\tau\hat{\xi}^{(0)}\big\Vert_{\infty}\right)+\underbrace{\frac{1}{1-\lambda_1}\big[(2-2\alpha + c_1)\epseval + (1-\alpha+c_2)\epsopt\big]}_{\eqqcolon \, C}\bigg]
\begin{bmatrix}
    \gamma\\
    1
\end{bmatrix}.
\end{align*}
Making use of the fact that $1 - \lambda_1 = (1-\alpha)(1-\gamma)$, we can conclude
\[
    C = \frac{1}{1-\gamma}\left[\left(2 + \frac{c_1}{1-\alpha}\right)\epseval + \left(1 + \frac{c_2}{1-\alpha}\right)\epsopt\right].
\]
The above bound essentially says that
\begin{align*}
& \big\Vert{Q}^{\star}_{\tau}-{Q}^{{(k+1)}}_{\tau} \big\Vert_{\infty}\le\gamma\left[\lambda_1^k\left( \big\Vert{Q}^{\star}_{\tau}-{Q}^{(0)}_{\tau} \big\Vert_{\infty}+2\alpha \big\Vert {Q}^{\star}_{\tau}-\tau\hat{\xi}^{(0)} \big\Vert_{\infty}\right)+C\right]
\end{align*}
%
and
\begin{align*}
	&\big\|Q_\tau^\star - \tau \hat{\xi}^{(k+1)} \big\|_{\infty}
	\le\lambda_1^k\left( \big\Vert{Q}^{\star}_{\tau}-{Q}^{(0)}_{\tau} \big\Vert_{\infty}+2\alpha \big\Vert {Q}^{\star}_{\tau}-\tau\hat{\xi}^{(0)} \big\Vert_{\infty}\right)+C.
\end{align*}

Turning to $ V_\tau^\star (s) - V_\tau^{(k+1)}(s)$, by a similar argument as \eqref{eq:V_bound_exact}, we have
\begin{equation*}
\begin{aligned}
	& V_\tau^\star (s) - V_\tau^{(k+1)}(s) \notag\\
	&= \big\langle Q_\tau^\star(s) - Q_\tau^{(k+1)}(s), \pi^{(k+1)}(s) \big\rangle + \Big[\tau(h_s(\pi^{(k+1)}(s)) - h_s(\pi_\tau^\star(s))) - \big\langle Q_\tau^\star(s), \pi^{(k+1)}(s) - \pi_\tau^\star(s) \big\rangle\Big]\\
	&= \big\langle Q_\tau^\star(s) - Q_\tau^{(k+1)}(s), \pi^{(k+1)}(s) \big\rangle + \tau D_{h_s}(\pi^{(k+1)}, \pi_\tau^\star; g_\tau^\star)\\
	&\le \Big\|Q_\tau^\star - Q_\tau^{(k+1)}\Big\|_\infty  + \tau D_{h_s}(\pi^{(k+1)}, \tilde\pi^{(k+1)}; \hat{\xi}^{(k+1)})\\
	&\qquad + \tau D_{h_s}(\tilde\pi^{(k+1)}, \pi_\tau^\star; g_\tau^\star) + \tau\big\langle \pi^{(k+1)}(s) - \tilde\pi^{(k+1)}(s) , \hat{\xi}^{(k+1)}(s) - g_\tau^\star(s)\big\rangle,
\end{aligned}
\end{equation*}
where the third step results from the standard three-point lemma. To control the second term, we rearrange terms in \eqref{eq:opt_oracle} and reach at
\begin{align*}
	\epsopt &\ge-\big\langle \hat{Q}_\tau^{(k)}(s),\pi^{(k+1)}(s) \big\rangle  +  \tau h_{s}\big(\pi^{(k+1)}(s)\big)+\frac{1}{\eta}D_{h_{s}}\big( \pi^{(k+1)}, \pi^{(k)}; \hat{\xi}^{(k)} \big) \notag \\
    & \qquad + \big\langle \hat{Q}_\tau^{(k)}(s),\tilde{\pi}^{(k+1)}(s) \big\rangle
    - \tau h_{s}\big(\tilde{\pi}^{(k+1)}(s)\big)-\frac{1}{\eta}D _{h_{s}}\big( \tilde{\pi}^{(k+1)}, \pi^{(k)}; \hat{\xi}^{(k)} \big)\\
    &= \big\langle \hat{Q}_\tau^{(k)}(s),\tilde{\pi}^{(k+1)}(s) - \pi^{(k+1)}(s)\big\rangle + \frac{1+\eta\tau}{\eta}\big(h_s(\pi^{(k+1)}(s)) - h_s(\tilde\pi^{(k+1)}(s))\big)\\
    &\qquad + \frac{1}{\eta}\big\langle \hat{\xi}^{(k)}(s),\tilde{\pi}^{(k+1)}(s) - \pi^{(k+1)}(s)\big\rangle\\
    &= \frac{1+\eta\tau}{\eta}D_{h_s}(\pi^{(k+1)}, \tilde\pi^{(k+1)}; \hat{\xi}^{(k+1)}).
\end{align*}
For the remaining terms, recall that $\hat{\xi}^{(k+1)} - c_s^{(k+1)} 1 \in \partial h_s(\tilde{\pi}^{(k+1)}(s))$ with some constant $c_s^{(k+1)}$. So we have
\begin{align*}
	&\tau D_{h_s}(\tilde\pi^{(k+1)}, \pi_\tau^\star; g_\tau^\star) + \tau\big\langle \pi^{(k+1)}(s) - \tilde\pi^{(k+1)}(s) , \hat{\xi}^{(k+1)}(s) - g_\tau^\star(s)\big\rangle\\
	& = \tau h_s(\tilde\pi^{(k+1)}(s)) - \tau h_s(\pi_\tau^\star(s))- \big\langle \tilde\pi^{(k+1)}(s) - \pi_\tau^\star(s), Q_\tau^\star(s)\big\rangle + \tau\big\langle \pi^{(k+1)}(s) - \tilde\pi^{(k+1)}(s) , \hat{\xi}^{(k+1)}(s) - g_\tau^\star(s)\big\rangle\\
	&\le \big\langle  \pi_\tau^\star(s) - \tilde\pi^{(k+1)}(s), Q_\tau^\star(s) - \tau \hat\xi^{(k+1)} \big\rangle - \big\langle \pi^{(k+1)}(s) - \tilde\pi^{(k+1)}(s) , Q_\tau^\star(s) -  \tau \hat{\xi}^{(k+1)}(s)\big\rangle\\
	&= \big\langle  \pi_\tau^\star(s) - \pi^{(k+1)}(s), Q_\tau^\star(s) - \tau \hat\xi^{(k+1)} \big\rangle\\
	&\le 2 \big\| Q_\tau^\star (s) - \tau\hat\xi^{(k+1)} (s) \big\|_\infty.
\end{align*}
Taken together, we conclude that
\begin{align*}
V_\tau^\star (s) - V_\tau^{(k+1)}(s) &\le \Big\|Q_\tau^\star - Q_\tau^{(k+1)}\Big\|_\infty + 2 \big\| Q_\tau^\star (s) - \tau\hat\xi^{(k+1)} (s) \big\|_\infty +\frac{\eta\tau }{1+\eta\tau}\epsopt\\
&\le (\gamma + 2)\left[\lambda_1^k\left( \big\Vert{Q}^{\star}_{\tau}-{Q}^{(0)}_{\tau} \big\Vert_{\infty}+2\alpha \big\Vert {Q}^{\star}_{\tau}-\tau\hat{\xi}^{(0)} \big\Vert_{\infty}\right)+C\right] +\frac{\eta\tau }{1+\eta\tau}\epsopt.
\end{align*}

Finally, plugging in the choices of $c_1$ and $c_2$ (cf. \eqref{eq:c_1c_2}), we have $C\le C_2$ when $\{h_s\}$ is convex, and $C \le C_3$ when $\{h_s\}$ is $1$-strongly convex w.r.t.~the $\ell_1$ norm. In addition, for the latter case, we can follow a similar argument as for \eqref{eq:pi_conv_l1} to demonstrate that
\begin{align*}
    \big\|\pi_\tau^\star(s) - \tilde{\pi}_\tau^{(k)}(s) \big\|_1 &\le \tau^{-1}\big((1-\alpha)\gamma+\alpha \big)^k\left( \big\Vert{Q}^{\star}_{\tau}-{Q}^{(0)}_{\tau} \big\Vert_{\infty}
	+ 2\alpha \big\Vert{Q}^{\star}_{\tau}-\tau\xi^{(0)} \big\Vert_{\infty}\right)
    + \tau^{-1}C_3,
\end{align*}
which taken together with \eqref{eq:pi_err_l1} gives
\begin{align*}
    \big\|\pi_\tau^\star(s) - \pi_\tau^{(k)}(s) \big\|_1 &\le \big\|\pi_\tau^\star(s) - \tilde{\pi}_\tau^{(k)}(s) \big\|_1 + \big\|\pi_\tau^{(k)}(s) - \tilde{\pi}_\tau^{(k)}(s) \big\|_1\\
    &\le \tau^{-1}\big((1-\alpha)\gamma+\alpha \big)^k\left(\big\Vert{Q}^{\star}_{\tau}-{Q}^{(0)}_{\tau} \big\Vert_{\infty}+2\alpha \big\Vert{Q}^{\star}_{\tau}-\tau\xi^{(0)} \big\Vert_{\infty}\right)\\
    &\quad + \tau^{-1}C_3 + \sqrt{\frac{2\eta\epsopt}{1+\eta\tau}}.
\end{align*}
This concludes the proof.

\section{Adaptive GPMD}
\label{sec:adaptive}
In this section, we present adaptive GPMD, an adaptive variant of GPMD that computes optimal policies of the original MDP without the need of specifying the regularization parameter $\tau$ in advance. 
 In a nutshell, the proposed adaptive GPMD algorithm is a stage-based algorithm. In the $i$-th stage, we execute GPMD (i.e., Algorithm~\ref{alg:GPMD}) with regularization parameter $\tau_i$ for $T_i+1$ iterations. 
 In what follows, we shall denote by $\pi_i^{(t)}$ and $\xi_i^{(t)}$ the $t$-th iterates in the $i$-th stage. At the end of each stage, Adaptive GPMD will halve the regularization parameter $\tau$, 
 and in the meantime,  double $\xi_i^{(T_i+1)}$ (i.e., the auxiliary vector corresponding to some subgradient up to global shift) and use it to as the initial vector $\xi_{i+1}^{(0)}$ for the next stage. 
 To ensure that $\xi_{i+1}^{(0)}(s)$ still lies within the set of subgradients $\partial h_s\big( \pi_{i+1}^{(0)}(s) \big)$ up to global shift, 
 we solve the sub-optimization problem \eqref{eq:ad ini} to obtain $\pi_{i+1}^{(0)}$ as the initial policy iterate for the next stage. The whole procedure is summarized in Algorithm~\ref{alg:AdGPMD}.

\begin{algorithm}[th]
\caption{Adaptive GPMD}
\label{alg:AdGPMD}
\textbf{Input:} learning rate $\eta>0$.\\
\textbf{Initialize} $\tau_0=1,\xi_0^{(0)}(s,a)=0$ for all $s\in\cS, a\in\A$. Choose $\pi_0^{(0)}$ to be the minimizer of the following problems:
\[
\pi_0^{(0)}(s)=\arg\min_{p\in\Delta(\A)}h_s(p),\qquad \forall s\in\cS.
\]

\For{stage $i=0,1,\cdots,$}{
    Call Algorithm~\ref{alg:GPMD} with regularization parameter $\tau_i$, learning rate $\eta$, and initialization $\pi_{i}^{(0)}$ and $\xi_{i}^{(0)}$. Obtain $\pi_{i}^{(T_i+1)}$ and $\xi_{i}^{(T_i+1)}$ with $T_i=\big\lceil\frac{1+\eta\tau_i}{(1-\gamma)\eta\tau_i}\log\frac{8}{1-\gamma}\big\rceil$, where $\lceil\cdot\rceil$ is the ceiling function.
    
    Set $\tau_{i+1}=\tau_i/2$, $\xi_{i+1}^{(0)}=2\xi_{i}^{(T_i+1)}$, and choose $\pi_{i+1}^{(0)}$ to be the minimizer of the following problems:
    \begin{align}
    \label{eq:ad ini}
    	\pi_{i+1}^{(0)}(s)=\arg\min_{p\in\Delta(\A)}-\langle\xi_{i+1}^{(0)}(s),p\rangle+h_s(p), \qquad \forall s\in\cS.
    \end{align}	
}
\end{algorithm}
To help characterize the discrepancy of the value functions due to regularization, we assume boundedness of the regularizers $\{h_s\}$ as follows.
\begin{assumption}
\label{ass:bounded}
Suppose that there exists some quantity $B>1$ such that
 $|h_s(p)|\leq B$ holds for all $p\in\Delta(\A)$ and all $s\in\cS$.
\end{assumption}

The following theorem demonstrates that Algorithm~\ref{alg:AdGPMD} is capable of finding an $\eps$-optimal policy for the unregularized MDP within an order of $\log \frac{1}{\eps}$ stages. 
To simplify notation, we abbreviate $Q_{\tau_i}^\star$, $Q_{\tau_i}^{\pi_i^{(T)}}$ and $V_{\tau_i}^{\pi_i^{(T)}}$ as $Q_{i}^\star$, $Q^{(T)}_{i}$ and $V^{(T)}_{i}$, respectively, as long as it is clear from the context. 
\begin{theorem}
\label{thm:adapt}
Suppose that Assumptions~\ref{assumption:h-inf} and \ref{ass:bounded} hold. For any learning rate $\eta>0$ and any stage $i \ge 0$, the iterates of Algorithm~\ref{alg:AdGPMD} satisfy
\[
\big\Vert Q^\star-Q^{\pi_{i}^{(T_i+1)}}\big\Vert_{\infty}\leq\frac{3\tau_i B}{1-\gamma} = \frac{3 B}{(1-\gamma)2^i}.
\]
\end{theorem}
As a direct implication of Theorem~\ref{thm:adapt},  it suffices to run Algorithm \ref{alg:AdGPMD} with $S = \BO(\log\frac{B}{(1-\gamma)\varepsilon})$ stages, resulting in a total iteration complexity of at most
\begin{align}
 \sum_{i=0}^ST_i=\BO\left(\left(\frac{1}{1-\gamma}\log\frac{B}{(1-\gamma)\varepsilon}+\frac{B}{(1-\gamma)^2\eta\varepsilon}\right)\log\frac{1}{1-\gamma}\right). 
	\label{eq:iteration-complexity-ad-GPMD}
\end{align}
In comparison, we recall from Theorem \ref{thm:linear_exact} that: directly running GPMD with regularization parameter $\tau = {(1-\gamma)\varepsilon}/{B}$ leads to an iteration complexity of
\begin{align}
\BO\Big(\Big(\frac{1}{1-\gamma} + \frac{B}{(1-\gamma)^2\eta\varepsilon}\Big)\log\frac{B}{(1-\gamma)\varepsilon}\Big). 
	\label{eq:iteration-complexity-non-ad-GPMD}
\end{align}
When focusing on the term $\widetilde{\mathcal{O}}(\frac{B}{(1-\gamma)^2\eta\varepsilon})$, \eqref{eq:iteration-complexity-ad-GPMD} improves upon \eqref{eq:iteration-complexity-non-ad-GPMD} by a factor of $\frac{\log\frac{B}{(1-\gamma)\varepsilon}}{\log\frac{1}{1-\gamma}}$.

\begin{proof}[Proof of Theorem~\ref{thm:adapt}]
To begin with, we make note of the fact that, for any $\tau,\tau'>0$,
\begin{equation}
\big\Vert Q^\star_{\tau}-Q^\star_{\tau'}\big\Vert_{\infty}=\Big\Vert\max_{\pi}Q^{\pi}_{\tau}-\max_{\pi}Q^{\pi}_{\tau'}\Big\Vert_{\infty}\leq\max_{\pi}\big\Vert Q^{\pi}_{\tau}-Q^{\pi}_{\tau'}\big\Vert_{\infty}\leq\frac{|\tau-\tau'|B}{1-\gamma}.
\label{eq:tau_diff}
\end{equation}
It then follows that
\begin{equation}
\begin{aligned}
\label{eq:adapt_bound_1}
\big\Vert Q^\star-Q^{\pi_{i}^{(T_i+1)}}\big\Vert_{\infty} & = \big\Vert Q^\star-Q_i^{\star}\big\Vert_{\infty}  + \big\Vert Q^{\pi_{i}^{(T_i+1)}} - Q_i^{(T_i+1)}\big\Vert_\infty + \big\Vert Q^\star_{i}-Q_{i}^{(T_i+1)} \big\Vert_\infty\\
& \le \frac{2\tau_i B}{1-\gamma}+\big\Vert Q^\star_{i}-Q_{i}^{(T_i+1)}\big\Vert_{\infty}.
\end{aligned}
\end{equation}

Next, we demonstrate how to control $\Vert Q^\star_{i}-Q_{i}^{(T_i+1)} \Vert_{\infty}$. 
The definition of $\pi_{i}^{(0)}$ implies the existence of some constant $c_s^{(i,0)}$ such that
\begin{align}
\xi_{i}^{(0)}(s,\cdot) - c_s^{(i,0)} 1  \in  \partial h_s\big( \pi_{i}^{(0)}(s) \big).
\end{align}
By invoking the convergence results of GPMD (cf. \eqref{eq:Qstar-UB-123-exact}), we  obtain: for all $i\geq0$, 
\begin{subequations}
	\label{eq:adapt_stage_conv}
	\begin{align}
		\big\Vert Q^\star_i-Q_{i}^{(T_i+1)} \big\Vert_{\infty}
		& \leq\gamma \big((1-\alpha_i)\gamma+\alpha_i \big)^{T_i}\left( \big\Vert Q^\star_i-{Q}_{i}^{(0)} \big\Vert_{\infty}+2\alpha_i\big\Vert Q^\star_i-\tau_i\xi_{i}^{(0)}\big\Vert_{\infty}\right), \\
		\big\|Q^\star_i - \tau_i \xi_{i}^{(T_i+1)} \big\|_\infty 
		&\le \big((1-\alpha_i)\gamma+\alpha_i \big)^{T_i}\left( \big\Vert Q^\star_i-{Q}_{i}^{(0)} \big\Vert_{\infty}+2\alpha_i\big\Vert Q^\star_i-\tau_i\xi_{i}^{(0)}\big\Vert_{\infty}\right),
	\end{align}
\end{subequations}
where $\alpha_i=\frac{1}{1+\eta\tau_i}$. 
To proceed, we follow similar arguments in \eqref{eq:V_bound_exact} and show that
\begin{align*}
	V^{\star}_{i}(s)-V^{(0)}_{i}(s) &=\frac{1}{1-\gamma} \mathop{\E}\limits_{s\sim d^{\pi}_{s'}}\Big[\left\langle\pi^{\star}_{\tau_i}(s)-\pi_i^{(0)}(s), Q^\star_{i}(s) \right\rangle + \tau_i\big(h_{s}(\pi_i^{(0)})-h_{s}(\pi^{\star}_{\tau_i}) \big)\Big]\\
&\le \frac{1}{1-\gamma}\mathop{\E}\limits_{s\sim d^{\pi}_{s'}}\Big[\left\langle\pi^{\star}_{\tau_i}(s)-\pi_{i}^{(0)}(s), Q^\star_{i}(s) \right\rangle - \tau\Big\langle\pi^{\star}_{\tau_i}(s) - \pi_{i}^{(0)}(s), \xi_{i}^{(0)}\Big\rangle\Big]\\
&\le \frac{2}{1-\gamma} \big\| Q^\star_i - \tau_i\xi_{i}^{(0)} \big\|_\infty,
\end{align*}
where the first step invokes the regularized performance difference lemma (Lemma \ref{lemma:perf_diff}).
It then follows that
\begin{align}
&\big\Vert Q^\star_i-{Q}_{i}^{(0)} \big\Vert_{\infty}\leq\gamma\big\Vert V^{\star}_{i}-V^{(0)}_{i} \big\Vert_{\infty} \leq\frac{2\gamma}{1-\gamma}\big\| Q^\star_i - \tau_i\xi_{i}^{(0)} \big\|_\infty.
\label{eq:adapt 2}
\end{align}
Substitution of \eqref{eq:adapt 2} into \eqref{eq:adapt_stage_conv} gives
\begin{subequations}
	\begin{align}
		\big\Vert Q^\star_i-Q_{i}^{(T_i+1)} \big\Vert_{\infty}
		& \leq \frac{2\gamma}{1-\gamma} \big((1-\alpha_i)\gamma+\alpha_i \big)^{T_i}\big\Vert Q^\star_i-\tau_i\xi_{i}^{(0)}\big\Vert_{\infty}, \\
		\big\|Q^\star_i - \tau_i \xi_{i}^{(T_i+1)} \big\|_\infty 
		&\le \frac{2}{1-\gamma} \big((1-\alpha_i)\gamma+\alpha_i \big)^{T_i}\big\Vert Q^\star_i-\tau_i\xi_{i}^{(0)}\big\Vert_{\infty}.
	\end{align}
	\label{eq:adapt_stage_conv_simplified}
\end{subequations}

Next, we aim to prove by induction that
$\big\Vert Q^\star_i-\tau_i\xi_{i}^{(0)}\big\Vert_{\infty}\leq\frac{2\tau_i B}{1-\gamma}$. Clearly, this claim holds trivially for the base case with $i=0$.  Next, supposing that the claim holds for some $i\ge 0$, we would like to prove it for $i+1$ as well. 
Towards this end, observe that
\begin{align*}
\big\|Q^\star_{i+1} - \tau_{i+1} \xi_{i+1}^{(0)} \big\|_\infty &= \big\|Q^\star_{i+1} - Q^\star_{i} \big\|_\infty + \big\|Q^\star_{i} - \tau_{i+1} \xi_{i+1}^{(0)} \big\|_\infty \\
&\leq\frac{\tau_{i+1}B}{1-\gamma}+\big\|Q^\star_{i} - \tau_i \xi_{i}^{(T_i+1)} \big\|_\infty\\
&\leq\frac{\tau_{i+1}B}{1-\gamma}+\frac{2}{1-\gamma}\big((1-\alpha_i)\gamma+\alpha_i \big)^{T_i}\big\Vert Q^\star_i-\tau_i\xi_{i}^{(0)}\big\Vert_{\infty}\\
&\le \frac{\tau_{i+1}B}{1-\gamma} \left(1+\frac{8}{1-\gamma} \left(1 - \frac{(1-\gamma)\eta\tau_i}{1+\eta\tau_i} \right)^{T_i} \right).
\end{align*}
When $T_i\ge\lceil\frac{1+\eta\tau_i}{\eta\tau_i(1-\gamma)}\log\frac{8}{1-\gamma}\rceil$, we arrive at $$\big\|Q^\star_{i+1} - \tau_{i+1} \xi_{i+1}^{(0)}\big\|_\infty \le \frac{2\tau_{i+1}B}{1-\gamma},$$ 
which verifies the claim for $i+1$. Substitution back into \eqref{eq:adapt_stage_conv_simplified} leads to
\begin{equation}
\big\Vert Q^\star_i-Q_{i}^{(T_i+1)} \big\Vert_{\infty}\le \frac{2\gamma}{1-\gamma} \left(1 - \frac{(1-\gamma)\eta\tau_i}{1+\eta\tau_i} \right)^{T_i}\frac{2\tau_i B}{1-\gamma}\le \frac{\tau_i B}{1-\gamma}.
\label{eq:adapt_bound_2}
\end{equation}
Combining \eqref{eq:adapt_bound_2} with \eqref{eq:adapt_bound_1} concludes the proof.
\end{proof}


%
%
\ifdefined\isarxivversion
\bibliographystyle{apalike}
\else
\fi
\bibliography{ref.bib}
\end{document}